\documentclass{article} 
\usepackage{iclr2026_conference,times}

\usepackage{amsmath,amsfonts,bm}

\def\eqref#1{equation~\ref{#1}}

\def\1{\bm{1}}

\DeclareMathAlphabet{\mathsfit}{\encodingdefault}{\sfdefault}{m}{sl}
\SetMathAlphabet{\mathsfit}{bold}{\encodingdefault}{\sfdefault}{bx}{n}

\usepackage{hyperref}
\usepackage{url}

\iclrfinalcopy

\usepackage{amsmath}
\usepackage{amsthm}
\usepackage{amssymb}
\usepackage{varwidth}
\usepackage{graphicx}
\usepackage{float}
\usepackage[all]{xy}
\usepackage{subcaption}
\usepackage{booktabs}
\usepackage{array}

\providecommand{\tabularnewline}{\\}

\floatstyle{ruled}
\newfloat{algorithm}{tbp}{loa}
\providecommand{\algorithmname}{Algorithm}
\floatname{algorithm}{\protect\algorithmname}

\theoremstyle{plain}
\newtheorem{thm}{\protect\theoremname}
\theoremstyle{plain}
\newtheorem{lem}{\protect\lemmaname}

\providecommand{\lemmaname}{Lemma}
\providecommand{\theoremname}{Theorem}

\title{Deep-ICE: The first globally optimal algorithm for minimizing 0–1
	Loss in two-Layer ReLU and Maxout networks}

\author{Xi He\thanks{Designed the core algorithms, provided theoretical proofs, conducted the main experiments, and wrote the manuscript.}\\
School of Computer Science\\
Peking University\\
Beijing, China \\
\texttt{xihe@pku.edu.cn} \\
\And
  Yi Miao\thanks{Implemented the CUDA version of Deep-ICE algorithm, and co-investigated the ordered generation and memory-free techniques.}\\
School of Computer Science\\
University of Birmingham\\
Birmingham, B15 2TT, UK \\
\texttt{yxm296@student.bham.ac.uk} \\
\And
Max A. Little\thanks{Initiated the project and provided supervision and critical feedback throughout the research and writing process.} \\
School of Computer Science\\
University of Birmingham\\
Birmingham, B15 2TT, UK\\
\texttt{maxl@mit.edu} \\
}

\begin{document}

\maketitle

\begin{abstract}
This paper introduces the first globally optimal algorithm for the
empirical risk minimization problem of two-layer maxout and ReLU networks,
i.e., minimizing the number of misclassifications. The algorithm has
a worst-case time complexity of $O\left(N^{DK+1}\right)$, where $K$
denotes the number of hidden neurons and $D$ represents the number
of features. It can be can be generalized to accommodate arbitrary
computable loss functions without affecting its computational complexity.
Our experiments demonstrate that the proposed algorithm provides provably
exact solutions for small-scale datasets. To handle larger datasets,
we introduce a heuristic method that reduces the data size to a manageable
scale, making it feasible for our algorithm. This extension enables
efficient processing of large-scale datasets and achieves significantly
improved performance in both training and prediction, compared to state-of-the-art approaches
(neural networks trained using gradient descent and support vector
machines), when applied to the same models (two-layer networks with
fixed hidden nodes and linear models). 

The artifacts of the Deep-ICE algorithm can be found in 
\url{https://github.com/XiHegrt/DeepICE-algorithm-artifacts}.
\end{abstract}

\section{Introduction}

In recent years, neural networks have emerged as an extremely useful supervised learning technique, developed from early origins in the perceptron learning algorithm for classification problems. This model has revolutionized nearly every scientific field involving data analysis and has become one of the most widely used machine learning techniques today. Our work focuses on developing \textit{interpretable models} for high-stakes applications, where even minor errors can lead to catastrophic consequences. For example, an incorrectly denied parole may result in innocent people suffering years of imprisonment due to racial bias \citep{kirchner2016machine}, poor bail decisions can lead to the release of dangerous criminals, and machine learning–based pollution models have misclassified highly polluted air as safe to breathe \citep{mcgough2018bad}. In such settings, it is crucial to deploy models that are both accurate and transparent. 

One effective way to achieve this is to identify the best interpretable model within a given hypothesis set—a task that is uniquely suited to global optimal (exact) algorithms. Two-layer networks possess \textit{rich expressivity}, capable of representing any continuous function \citep{kolmogorov1957representations}, while remaining \textit{interpretable}\footnote{Interpretability is a domain-specific notion, so there cannot be an all-purpose definition. As \cite{rudin2019stop} noted ``Usually, however, an interpretable machine learning model is constrained in model form so that it is either useful to someone, or obeys structural knowledge.'' We claim 2-layer ReLU/Maxout networks are interpretable because: 1.\textbf{ Shallow architecture enables direct inspection}, a 2-layer neural network has a simple, transparent structure. The output is just a linear combination of these hidden unit activations. 2. \textbf{Geometric interpretation of ReLU/Maxout network is clear}, with nonlinear activations like ReLU, each hidden neuron represents a hyperplane decision boundary in the input space. The network, therefore, partitions the input space into piecewise linear regions. } since the output is a linear combination of hidden units. Consequently, the empirical risk minimization (ERM) problem for two-layer networks with ReLU or Maxout activation functions is not only practically useful but also theoretically significant, as it provides a foundation for understanding deep networks.

However, finding the ERM solution of a neural network remains extremely challenging.
\citet{goel2020tight} showed that minimizing the training error of
two-layer ReLU networks under squared loss is NP-hard, even in the
realizable setting (i.e., determining whether zero misclassification
is achievable). This result was later extended to $L^{p}$ loss with
$0\leq p<\infty$ \citep{froese2022computational,hertrich2022facets}.
In practice, this difficulty is further compounded when optimizing
discrete loss functions, such as the 0-1 loss (count the number of
misclassification), since the ultimate goal typically involves comparing
classification accuracy. Even in the simplest case—linear classification
using a single hyperplane—the problem of minimizing discrete losses
such as the 0-1 loss is NP-hard. The best-known exact algorithm for
0-1 loss linear classification has a worst-case time complexity of
$O\left(N^{D+1}\right)$, where $N$ is the number of data $D$ is
the number of features \citep{he2023efficient}.

Nevertheless, since neural networks (NNs) have finite VC-dimension
\citep{bartlett2019nearly}, they can, in principle, be trained exactly
in polynomial time \citep{mohri2012foundations}. The closest related
work is that of \citet{arora2016understanding}, who proposed a one-by-one
enumeration strategy to train a two-layer ReLU NN to global optimality
for convex objective functions. \citet{hertrich2022facets} later
extended their result to concave loss functions. However, both studies
provide only pseudocode and a vague complexity analysis, without publicly
available implementations or empirical validation. Moreover, they
do not show how to enumerate the hyperplane partitions; instead, they
assume these partitions are given.

\citet{arora2016understanding} further claim, somewhat ambiguously,
that their algorithm has a complexity of $O\left(2^{K}N^{DK}\mathit{poly}\left(N,D,K\right)\right)$
for a two-layer ReLU network with $K$ hidden neurons with respect to $N$ data points in $\mathbb{R}^D$. The term ``$\mathit{poly}\left(N,D,K\right)$''
is not explicitly defined; it refers to the complexity of solving a \textit{convex quadratic programming problem} with $K$ and $D$ variables and $N\times K$ constraints, and is therefore polynomial in $N$, which we denote as $O(C_{1}N^{C_{2}})$.
Therefore, \citet{arora2016understanding}'s algorithm involves not
only extremely large exponents ($D\times K+C_2$) but also formidable
constant factors ($2^{K}\times C_1$). 

As a result of the ambiguous algorithmic description and complexity analysis, the methods proposed by  \citet{arora2016understanding} and \citet{hertrich2022facets} appear more like a \textit{conjecture}—suggesting the existence of a polynomial-time algorithm—rather than practically executable solutions. The prohibitive complexity in both the exponent and constant terms renders their algorithms impractical even for small-scale problems. This is further highlighted by the absence of any implementation in the \textit{eight years} since their initial publication. Moreover, their algorithms are limited to convex loss functions, while the fundamental objective of classification is to minimize the number of misclassified instances, i.e., the 0-1 loss.

Interestingly, \citet{bai2023efficient} show that training a ReLU
network with an $L^{2}$-regularized convex loss objective can be
reformulated as a convex program and solved using a general-purpose
solver. However, a major limitation of such solvers is their unpredictable
computational complexity. Moreover, \citet{bai2023efficient} consider
a much simpler problem than optimizing the 0–1 loss—the original objective
in classification—whose discrete nature makes it substantially more
difficult to optimize. Empirical results from \citet{Xi_Exact_0-1_loss_2023}
further demonstrate that even for the simplest network—the linear
classifier—using a general-purpose solver to optimize the 0–1 loss
exhibits highly unpredictable behavior and can incur exponential complexity,
even in situations where a polynomial-time solution exists.

To address these limitations, this paper introduces the\emph{ first
	globally optimal algorithm for minimizing 0–1 loss in two-Layer ReLU
	and Maxout networks}. Our contributions can be summarized as follows:
\begin{itemize}
	\item \textbf{First optimal algorithm for 0-1 loss}. We present the first
	optimal algorithm for the empirical risk minimization problem of two-layer
	maxout and ReLU networks under the 0–1 loss. In contrast, prior method
	\citet{arora2016understanding,hertrich2022facets} are restricted
	to convex loss functions, which are comparatively easier to optimize
	than discrete losses such as the 0–1 loss. Our algorithm extends to
	any computable loss function by adapting the results of \citet{he2023efficient}
	without increasing worst-case complexity.
	\item \textbf{Two versions of the DeepICE algorithm}. Existing methods \citep{arora2016understanding,hertrich2022facets}
	rely on hidden assumptions. In practice, generating hyperplane predictions
	requires substantial computation, yet their pseudocode initializes
	all partitions directly without such effort. Moreover, their complexity
	analyses are ambiguous, hindering both understanding and reproducibility.
	Consequently, no implementation has emerged in the eight years since
	their publication. In contrast, by leveraging a general formalism,
	our algorithm admits a concise and unambiguous definition in a single
	equation (\ref{fig:voicemap}). We further provide two variants of
	the DeepICE algorithm: the \textbf{sequential version} (Algorithm
	\ref{alg:Deep-ICE-sequential}) which reuses hyperplane predictions
	via memoization, and the \textbf{divide-and-conquer} version (Algorithm
	\ref{alg:Deep-ICE-D=000026C}), which supports parallelization without
	inter-processor communication.
	\item \textbf{Improved computational complexity}. Our algorithm achieves
	a complexity of $O\left(2^{K-1}\times N^{DK+1}+N^{D}\times D^{3}\right)$,
	substantially better than the approaches of \citet{arora2016understanding}
	and \citet{hertrich2022facets}, which require $O\left(2^{K}\times C_{1}\times N^{DK+C_{2}}\right)$
	in both the best and worst cases. In addition, our algorithm exhibits\emph{
		significantly smaller constant factors}. This efficiency enables exact
	solutions for datasets with formidable combinatorial complexity—for
	example, the problem in Figure \ref{fig:voicemap}, which involves
	\emph{122,468,448,960 }configurations, can be solved within \textbf{minutes}
	using our CUDA implementation.
	\item \textbf{Robustness}. When combined with heuristics for large-scale
	problems, and training accuracy is significantly higher than that of SVMs or DNNs trained
	with gradient descent, our algorithm demonstrates strong out-of-sample performance. This result challenges the widely held belief
	that optimal algorithms necessarily overfit the training data.
\end{itemize}
The remainder of this paper is organized as follows. Section 2 presents
our main theoretical contributions: Section 2.1 introduces the necessary
background; Section 2.2 explains how geometric insights simplify the
combinatorics of the problem; Section 2.3 describes the construction
of an efficient recursive nested combination generator, which is the
core component of the Deep-ICE algorithm; and Section 2.4 presents
the fusion law for the Deep-ICE algorithm. Section 3 reports empirical
results. Finally, Section 4 summarizes our contributions and outlines
directions for future research.

\section{Theory}

\subsection{Theory of lists}

\paragraph{List homomorphisms}

The \emph{cons-list} is defined as $\mathit{ListR}\left(A\right)=\left[\;\right]\mid A:\mathit{ListR}\left(A\right)$;
that is, a list is either an empty list $\left[\;\right]$ or a pair
consisting of a head element $a:A$ and a tail $x:ListR\left(A\right)$,
concatenated using the cons operator $:$. For example, $1:\left[2,3\right]=\left[1,2,3\right]$.
This cons-list corresponds to the singly linked list data structure
in imperative languages. The key difference here is that we are referring
to the model of the data structure—i.e., the datatype—rather than
a specific implementation. There is a corresponding \emph{homomorphism}
over the cons-list datatype, which is a \emph{structure-preserving}
\emph{map} satisfying

\begin{equation}
	\begin{aligned}h & \left(\left[\;\right]\right)=alg_{1}\left(\left[\;\right]\right)\\
		h & \left(a:x\right)=alg_{2}\left(a,h\left(x\right)\right)
	\end{aligned}
\end{equation}
where $h:ListR\left(A\right)\to X$ . In other words, a homomorphism
over a cons-list is simply a recursion that sequentially combines
each element $a$ with the accumulated result $h\left(x\right)$ using
the algebra $alg$.

Alternatively, another list model called the \emph{join-list} is defined
as $\mathit{ListJ}\left(A\right)=\left[\;\right]\mid A\mid\mathit{ListJ}\left(A\right)\cup\mathit{ListJ}\left(A\right)$.
A join-list is either empty, a singleton list, or the result of joining
two sublists. The join operator $\cup$ is associative, i.e., for
any $x,y:\mathit{ListJ}$, we have: $x\cup\left[a\right]\cup y=\left(x\cup\left[a\right]\right)\cup y=x\cup\left(\left[a\right]\cup y\right)$.
The corresponding homomorphism over join-lists is a structure-preserving
map defined as
\begin{equation}
	\begin{aligned}h & \left(\left[\;\right]\right)=alg_{1}\left(\left[\;\right]\right)\\
		h & \left(\left[a\right]\right)=alg_{2}\left(\left[a\right]\right)\\
		h & \left(x\cup y\right)=alg_{3}\left(h\left(x\right),h\left(y\right)\right)
	\end{aligned}
\end{equation}
An example of a join-list homomorphism that computes the length of
a list uses the definitions $alg_{1}\left(\left[\;\right]\right)=0$,
$alg_{2}\left(a\right)=1$, and $alg_{3}\left(x\cup y\right)=h\left(x\right)+h\left(y\right)$.

\paragraph{Fusion laws}

An important principle associated with both cons-list and join-list
homomorphisms is the \emph{fusion} \emph{law}, stated in the following
two theorems. Its correctness can be verified either by using induction
\citep{bird2020algorithm} or universal property \citep{bird1996algebra}.
For brevity, we omit the proofs here.
\begin{thm}
	Fusion law for the cons-list.\emph{ Let $f$ be a function and let
		$h$ and $g$ be two cons-list homomorphisms defined by the algebras
		$alg$ and $alg^{\prime}$, respectively. The fusion law states that
		$f\circ h=g$ if 
		\begin{equation}
			f\left(alg\left(a,h\left(x\right)\right)\right)=alg^{\prime}\left(a,h\left(x\right)\right).\label{eq: cons-list fusion condition}
		\end{equation}
	}
\end{thm}
Similarly, the fusion condition for the join-list is defined as following.
\begin{thm}
	Fusion law for the join-list.\emph{ Let $f$ be a function and let
		$h$ and $g$ be two join-list homomorphisms defined by the algebras
		$alg$ and $alg^{\prime}$ respectively. The fusion law states that
		$f\circ h=g$ if}
	
	\begin{equation}
		f\left(alg\left(\left(h\left(x\right),h\left(y\right)\right)\right)\right)=alg^{\prime}\left(f\left(h\left(x\right)\right),f\left(h\left(y\right)\right)\right).\label{eq: join-list fusion condition}
	\end{equation}
	\emph{In point-free style}\footnote{Point-free is a style of defining functions without explicitly mentioning
		their arguments.}\emph{, this can be expressed more succinctly as $f\circ alg=alg^{\prime}\circ f\times f$,
		where $f\times g\left(x,y\right)=\left(f\left(x\right),g\left(y\right)\right)$.\label{thm: join-list fusion law}}
\end{thm}
Equations (\ref{eq: cons-list fusion condition}) and (\ref{eq: join-list fusion condition})
are referred to as the \emph{fusion} \emph{condition}, which forms the basis for proving the correctness of the derived algorithm. 

\subsection{Problem specification}

Assume we are given a data list $ds=\left[\boldsymbol{x}_{1},\boldsymbol{x}_{2}\ldots,\boldsymbol{x}_{N}\right]:\left[\mathbb{R}^{D}\right]$,
where the points are in general position (i.e., no $d+1$ points lie
on the same $\left(d-1\right)$-dimensional affine subspace of $\mathbb{R}^{D}$),
and $D\geq2$. We associate each data point $\boldsymbol{x}_{n}$
with a true label $t_{n}\in\left\{ 1,-1\right\} $. We extend the
ReLU activation function to vectors $\boldsymbol{x}\in\mathbb{R}^{D}$
via an entry-wise operation $\sigma\left(\boldsymbol{x}\right)=\left(\max\left(0,x_{1}\right),\max\left(0,x_{2}\right),\ldots,\max\left(0,x_{D}\right)\right)$.

Now, consider a two-layer feedforward ReLU NN with $K$ hidden units.
Each hidden node is associated with an affine transformation $f_{\boldsymbol{w}_{k}}:\mathbb{R}^{D+1}\to\mathbb{R}$,
which corresponds to a homogeneous hyperplane $h_{k}$ with normal
vector $\boldsymbol{w}_{k}\in\mathbb{R}^{D+1},\forall k\in\mathcal{K}=\left\{ 1,2,\ldots,K\right\} $.
These $K$ affine transformations can be represented by a single affine
transformation $f\left(\boldsymbol{W}_{1}\right):\mathbb{R}^{D+1}\to\mathbb{R}^{K}$,
where $\boldsymbol{W}_{1}\in\mathbb{R}^{K\times\left(D+1\right)}$,
with rows given by the vectors $\boldsymbol{w}_{k}$, i.e., $\boldsymbol{W}_{1}^{T}=\left(\boldsymbol{w}_{1},\boldsymbol{w}_{2},\ldots,\boldsymbol{w}_{K}\right)$.
The output of the hidden layer is then passed through the ReLU activation,
followed by a linear transformation $f\left(\boldsymbol{W}_{2}\right):\mathbb{R}^{K}\to\mathbb{R}$,
where $\boldsymbol{W}_{2}=\left(\alpha_{1},\alpha_{2},\ldots,\alpha_{K}\right)$
are the weights connecting the hidden layer to the output node. Thus,
the decision function $f_{\text{ReLU}}$ implemented by the network
is given by
\begin{equation}
	f_{\text{ReLU}}\left(\boldsymbol{W}_{1},\boldsymbol{W}_{2}\right)=f\left(\boldsymbol{W}_{2}\right)\circ\sigma\circ f\left(\boldsymbol{W}_{1}\right).\label{eq: 2-layer network decision function}
\end{equation}
Alternatively, instead of applying the ReLU activation function $\sigma$
followed by a linear transformation $f\left(\boldsymbol{W}_{2}\right)$,
the rank-$K$ maxout network with a single maxout neuron, replaces
both components with a maximum operator $\max_{\mathcal{K}}:\mathbb{R}^{K}\to\mathbb{R}$.
The resulting decision function is given by
\begin{equation}
	f_{\text{maxout}}\left(\boldsymbol{W}_{1}\right)=\max_{\mathcal{K}}\circ f\left(\boldsymbol{W}_{1}\right)
\end{equation}
Let $\mathcal{S}$ denote the \emph{combinatorial search space}. For
the ReLU and maxout networks, we define the configurations as $s_{\text{ReLU}}=\left(\boldsymbol{W}_{1},\boldsymbol{W}_{2}\right)\in\mathcal{S}_{\text{ReLU}}$
and $s_{\text{maxout}}=\boldsymbol{W}_{1}\in\mathcal{S}_{\text{maxout}}$,
respectively. The ERM problem for both network types can then be formulated
as the following optimization
\begin{equation}
	s^{*}=\underset{s\in\mathcal{S}}{\text{argmin}}E_{\text{0-1}}\left(s\right),\label{eq: ReLU ERM}
\end{equation}
where $E_{\text{0-1}}\left(s_{\text{ReLU}}\right)=\sum_{n\in\mathcal{N}}\boldsymbol{1}\left[\textrm{sign}\left(f_{\text{ReLU}}\left(\boldsymbol{W}_{1},\boldsymbol{W}_{2},\bar{\boldsymbol{x}}_{n}\right)\right)\neq t_{n}\right]$
for ReLU network, and $E_{\text{0-1}}\left(s_{\text{maxout}}\right)=\sum_{n\in\mathcal{N}}\boldsymbol{1}\left[\textrm{sign}\left(f_{\text{maxout}}\left(\boldsymbol{W}_{1},\bar{\boldsymbol{x}}_{n}\right)\right)\neq t_{n}\right]$
for maxout networks. In the following discussion, we primarily focus
on the maxout network, as an efficient speed-up technique is available
in this setting. Unless otherwise stated, when $E_{\text{0-1}}$ is
used it refers to the objective function for the maxout network by
default. Although our algorithm is compatible with any computable
objective function, to enable future acceleration strategies, it is
beneficial to restrict the choice of objective to be a monotonic linear
function of the form: $E_{\text{0-1}}\left(s_{\text{ReLU}}\right)=\sum_{n\in\mathcal{N}}L\left(\bar{\boldsymbol{x}}_{n},t_{n}\right)$,
such that $L\left(\bar{\boldsymbol{x}}_{n},t_{n}\right)\geq0$.

\paragraph{An exhaustive search specification}

Due to the \emph{distributivity} of the ReLU activation function—that
is, $\max\left(0,ab\right)=a\max\left(0,b\right)$, for $a\geq0$—the
decision function introduced by the two-layer ReLU network (\ref{eq: 2-layer network decision function})
can be rewritten as
\begin{equation}
	f_{\text{ReLU}}\left(\boldsymbol{W}_{1},\boldsymbol{W}_{2},\boldsymbol{x}\right)=\sum_{k\in\mathcal{K}}\alpha_{k}\max\left(0,\boldsymbol{w}_{k}\bar{\boldsymbol{x}}\right)=\sum_{k\in\mathcal{K}}z_{k}\max\left(0,\left|\alpha_{k}\right|\boldsymbol{w}_{k}\bar{\boldsymbol{x}}\right),\label{eq: relu decision funciton}
\end{equation}
where $\bar{\boldsymbol{x}}=\left(\boldsymbol{x},1\right)\in\mathbb{R}^{D+1}$
and $z_{k}\in\left\{ 1,-1\right\} $.

Similarly, the point-wise definition of the rank-$K$ maxout neuron
are defined as
\begin{equation}
	f_{\text{MO}}\left(\boldsymbol{W}_{1},\boldsymbol{x}\right)=\underset{k\in\mathcal{K}}{\text{max}}\left(\boldsymbol{w}_{k}\bar{\boldsymbol{x}}\right)\label{eq:maxout decision function}
\end{equation}
The decision function for a two-layer maxout network are simply the
linear combination of maxout neurons: $f_{\text{maxout}}\left(\boldsymbol{W}_{1},\boldsymbol{W}_{2},\boldsymbol{x}\right)=\sum_{k\in\mathcal{K}}\alpha_{k}\left(f_{\text{MO}}\left(\boldsymbol{W}_{1},\boldsymbol{x}\right)\right)$.

From a combinatorial perspective, the direction of the normal vector
does not affect the geometric definition of its associated hyperplane.
Therefore, equations (\ref{eq: relu decision funciton}) and (\ref{eq:maxout decision function})
indicate that the decision boundary of a \emph{two-layer ReLU} or
a single \emph{rank-$K$ maxout} neuron are fundamentally governed
by a $K$-combination of hyperplanes, and then combinations of hyperplanes
are composed again to form deep neural network. Although the set of
all possible hyperplanes in $\mathbb{R}^{D}$ appears to exhibit infinite
combinatorial complexity—since each hyperplane is parameterized by
a continuous-valued normal vector $\boldsymbol{w}_{k}$—the finiteness
of the dataset imposes a crucial constraint: \textbf{only a finite
	number of distinct data partitions} can be induced by these hyperplanes.
This observation introduces a natural notion of \textbf{equivalence}
\textbf{classes} over the space of hyperplanes, where two hyperplanes
are considered in the same equivalence class if they induce the same
partition over the dataset.

Indeed, according to the 0-1 loss linear classification theorem given
by \citet{he2023efficient}, when optimizing the 0-1 loss (i.e., minimizing
the number of misclassified data points), a hyperplanes in $\mathbb{R}^{D}$
an be characterized as the $D$-combinations of data points. Specifically,
each critical hyperplane corresponds to the affine span of $D$ data
points, leading to a total of $\left(\begin{array}{c}
	N\\
	D
\end{array}\right)=O\left(N^{D}\right)$ possible hyperplanes. This result implies that although the parameter
space is continuous, the effective combinatorial complexity of the
0-1 loss classification problem is polynomial in $N$ (for fixed $D$).
Each two-layer network with $K$ hidden neurons induces up to $2^{K}$
distinct partitions of the input space, determined by $2^{K}$ possible
directions of the normal vectors. These configurations can be encoded
as a length-$K$ binary assignment $\mathit{asgn}=\left(a_{1},\ldots a_{K}\right)\in\left\{ 1,-1\right\} ^{K}$.
Accordingly, a two-layer ReLU or maxout network can be characterized
by the pair $\mathit{cnfg}=\left(nc,\mathit{asgn}\right):\left(\mathit{NC},\left\{ 1,-1\right\} ^{K}\right)$,
where $\mathit{nc}:\mathit{NC}=\left[\left[\mathbb{R}^{D}\right]\right]$denotes
a nested combination, representing a $K$-combination of hyperplanes.

Thus, the combinatorial search space of a two-layer NN, denoted $\mathcal{S}\left(N,K,D\right)$
consists of the \emph{Cartesian product} of all possible$K$-combinations
of hyperplanes and the $2^{K}$ binary assignments.A provably correct
algorithm for solving the ERM problem of the two-layer network can
be constructed by exhaustively exploring all configurations in $\mathcal{S}\left(N,K,D\right)$
and selecting the network that minimizes the 0-1 loss. This procedure
is formally specified as
\begin{equation}
	\mathit{DeepICE}\left(D,K\right)=min_{\text{0-1}}\left(K\right)\circ\mathit{eval}\left(K\right)\circ\mathit{cp}\left(\mathit{basgns}\left(K\right)\right)\circ\mathit{\mathit{nestedCombs}}\left(D,K\right)\label{Problem specification}
\end{equation}
where $\mathit{DeepICE}\left(D,K\right):\left[\mathbb{R}^{D}\right]\to\left(NC,\left\{ 1,-1\right\} ^{K}\right)\times\mathit{Css}\times\mathit{NCss}$,
and $\mathit{NCss}=\left[\left[\left[\left[\mathbb{R}^{D}\right]\right]\right]\right]$
and $\mathit{Css}=\left[\left[\left[\mathbb{R}^{D}\right]\right]\right]$,
represent nested combinations and combinations, respectively. For
the parallelization concerns, $\mathit{DeepICE}\left(D,K\right)$
returns not only the optimal configuration for the input dataset $\mathcal{D}$
but also the intermediate representations $\mathit{NCss}$ and $\mathit{Css}$.
In the specification above, the input list $xs:\left[\mathbb{R}^{D}\right]$
is left implicit. The function $\mathit{\mathit{DeepICE}}\left(D,K,\mathit{ds}\right)$
generates \emph{all} \emph{possible} $K$-combinations of hyperplanes
($K$-hidden neuron networks) by and $\mathit{basgns}\left(K\right)$
produces all binary sign assignments of length $K$. These are combined
using the \emph{Cartesian} \emph{product} operator $\mathit{cp}\left(x,y\right)=\left[\left(a,b\right)\mid a\leftarrow x,b\leftarrow y\right]$.
Each resulting network is then evaluated by $\mathit{eval}\left(K\right)$,
which computes the objective value by considering all $2^{K}$ possible
orientations of the hyperplanes and selecting the best. Finally, $min_{\text{0-1}}\left(K\right)$
selects the configuration that minimizes the 0-1 loss.

In \emph{constructive} \emph{algorithmics} community \citep{bird1996algebra},
programs are initially defined as provably correct specifications,
such as—(\ref{Problem specification})—from which efficient implementations
are derived using algebraic laws like fusion. Efficiency arises both
from applying fusion transformations and from designing efficient
generators. To the best of our knowledge, no prior work has explored
generators for nested combinations. Moreover, fusion requires that
the generator be a recursive homomorphism—such as a cons-list or join-list
homomorphism. This precludes the \emph{non-recursive}, one-by-one
generation approach of \citet{arora2016understanding} which offers
opportunity for the application of acceleration techniques.

The key contribution of this paper is the development of an efficient
recursive nested combination generator, $\mathit{nestedCombs}\left(D,K,xs\right)$,
defined over a join-list homomorphism, making it amenable to fusion.
The generator is tailored for efficient vectorized and parallelized
implementations, making it ideal for GPU execution. We further demonstrate
that $min_{\text{0-1}}$, $\mathit{eval}$, and $\mathit{cp}$ are
all fusable with this generator. Additionally, the algorithm eliminates
the need for an initialization step to pre-store all hyperplanes and
continuously produces candidate solutions during runtime, allowing
approximate solutions to be obtained before the algorithm completes.

\subsection{An efficient nested combination generator join-list}

The first step for constructing an efficient nested combinations generator
requires the design of an efficient $K$-combination generator first.
Previously, \citet{he2024ekm} proposed an efficient combination generator,
$\mathit{kcombs}$, based on a join-list homomorphism, which we extend
to develop a nested combination generator. 

The \emph{nested combination-combination} generator is specified
as following

\begin{equation}
	\mathit{nestedCombs}\left(D,K\right)=\left\langle \mathit{setEmpty}\left(D\right),\mathit{kcombs}\left(K\right)\circ!!\left(D\right)\right\rangle \circ\mathit{kcombs}\left(D\right)\label{nested combs specification}
\end{equation}
where $\left\langle f,g\right\rangle \left(a\right)=\left(f\left(a\right),g\left(a\right)\right)$,
and $!!\left(D,\mathit{xs}\right)$ denotes indexing into the $D$th
element of the list $\mathit{xs}$. Equation (\ref{nested combs specification})
has the type $\mathit{nestedCombs}:Int\times Int\times\left[\mathbb{R}^{D}\right]\to\left(Css,NCss\right)$.
It first generates all possible $D$-combinations, and then all size
$D$-combinations which are then used to construct $K$-combinations.
Once this process is complete, the $D$-combinations are no longer
needed and are eliminated by applying $\mathit{setEmpty}\left(D\right)$
, which sets the $D$th element of the list to an empty value.

Although the specification in (\ref{nested combs specification})
is correct, it requires storing the intermediate result returned by
$\mathit{kcombs}\left(D,ds\right)$, which has a size of $O\left(N^{D}\right)$.
Storing all these combinations is both memory-intensive and inefficient.
Instead, if we can \emph{fuse} the function $\left\langle \mathit{setEmpty}\left(D\right),\mathit{kcombs}\left(K\right)\circ\left(!!D\right)\right\rangle $
directly into the $\mathit{\mathit{kcombs}}\left(D\right)$ generator,
the nested combination generator can be redefined as a single recursive
process. This transformation enables incremental generation of nested
combinations, eliminating the need to materialize all combinations
in advance. According to the fusion law \ref{thm: join-list fusion law},
this requires constructing an algebra $\mathit{nestedCombsAlg}$ that
satisfies the following fusion condition
\begin{equation}
	f\circ\mathit{\mathit{kcombsAlg}\left(D\right)}=\mathit{nestedCombsAlg}\left(D,K\right)\circ f\times f
\end{equation}
where $f=\left\langle \mathit{setEmpty}\left(D\right),\mathit{kcombs}\left(K\right)\circ\left(!!D\right)\right\rangle $,
and the definition of $\mathit{kcombsAlg}$ can be found in \citep{he2024ekm}

The derivation of $\mathit{\mathit{nestedCombsAlg}}\left(D,K\right)$
for the empty and singleton cases is relatively straightforward. Since
we assume $D\geq2$, no nested combinations can be constructed in
these cases. For the recursive case—i.e., the third pattern in the
join-list homomorphism—we show that the fusion condition holds when
this third pattern of $\mathit{\mathit{nestedCombsAlg}}\left(D,K\right)$
is defined as 
\begin{align}
	\Big\langle & \mathit{setEmpty}(D) \circ \mathit{KcombsAlg}(K) \circ \mathit{Ffst},\quad \nonumber \\
	& \mathit{KcombsAlg}(K) \circ \Big\langle 
	\mathit{Kcombs}(K) \circ !!\left(D\right) \circ \mathit{KcombsAlg}(D) \circ \mathit{Ffst},\,
	\mathit{KcombsAlg}(K) \circ \mathit{Fsnd}
	\Big\rangle 
	\Big\rangle, \label{nested combs gen-abstract}
\end{align}
where $\mathit{Ffst}\left(\left(a,b\right),\left(c,d\right)\right)=\left(a,c\right)$,
$\mathit{Fsnd}\left(\left(a,b\right),\left(c,d\right)\right)=\left(b,d\right)$.
The proof of the fusion condition is rather complex; for readability,
the complete proof is provided in Appendix \ref{subsec: Proof of nested generator}.
Therefore, we can implement $\mathit{nestedCombsAlg}\left(D,K\right)$
as

\begin{equation}
	\begin{aligned}\mathit{nestedCombsAlg}_{1} & \left(d,k,\left[\;\right]\right)=\left(\left[\left[\left[\;\right]\right]\right],\left[\left[\left[\;\right]\right]\right]\right)\\
		\mathit{nestedCombsAlg}_{2} & \left(d,k,\left[x_{n}\right]\right)=\left(\left[\left[\left[\;\right]\right],\left[\left[x_{n}\right]\right]\right],\left[\left[\left[\;\right]\right]\right]\right)\\
		\mathit{nestedCombsAlg}_{3} & \left(d,k,\left(css_{1},ncss_{1}\right),\left(css_{1},ncss_{1}\right)\right)=\left(\mathit{setEmpty}\left(D,css\right),ncss\right),
	\end{aligned}
\end{equation}
where $css=\mathit{kcombsAlg}\left(D,css_{1},css_{2}\right)$, and
$ncss$ is defined as 
\begin{equation}
	ncss=\begin{cases}
		\left[\left[\left[\;\right]\right]\right] & css!!\left(D\right)=\left[\;\right]\\
		\mathit{kcombsAlg}\left(K,\mathit{kcombsAlg}\left(K,ncss_{1},ncss_{2}\right),\mathit{kcombs}\left(K,css!!\left(D\right)\right)\right) & \text{otherwise}
	\end{cases},
\end{equation}
Thus an efficient recursive program for $nestedCombs$ is defined
as the following join-list homomorphism
\begin{equation}
	\begin{aligned}\mathit{netedCombs} & \left(D,K,\left[\;\right]\right)=\mathit{netedCombsAlg_{1}}\left(D,K,\left[\;\right]\right)\\
		\mathit{netedCombs} & \left(D,K,\left[x_{n}\right]\right)=\mathit{netedCombsAlg_{2}}\left(D,K,\left[x_{n}\right]\right)\\
		\mathit{netedCombs} & \left(D,K,xs\cup ys\right)=\\
		&\mathit{netedCombsAlg_{3}}\left(D,K\mathit{,netedCombs}\left(D,K,xs\right),\mathit{netedCombs}\left(D,K,ys\right)\right),
	\end{aligned}
\end{equation}
Informally, the function $\mathit{\mathit{nestedCombsAlg}}\left(D,K\right)$
first takes as input $\left(\left(\mathit{Css},\mathit{NCss}\right),\left(\mathit{Css},\mathit{NCss}\right)\right)$
which is returned by $f\times f$. The combination set is updated
using the composition $\mathit{setEmpty}\left(D\right)\circ\mathit{KcombsAlg}\left(K\right)\circ\mathit{Ffst}$
where the first elements of the tuple are updated, and the $D$-combinations
are cleared. At the same time, the function $\left\langle \mathit{Kcombs}\left(K\right)\circ!!\left(D\right)\circ\mathit{KcombsAlg}\left(D\right)\circ\mathit{Ffst},\mathit{KcombsAlg}\left(K\right)\circ\mathit{Fsnd}\right\rangle :\left(\left(\mathit{Css},\mathit{NCss}\right),\left(\mathit{Css},\mathit{NCss}\right)\right)\to\left(\mathit{NCss},\mathit{NCss}\right)$
updates the combinations and nested combinations in the tuple, respectively.
The newly generated $D$-combinations are then used to produce new
nested combinations. Finally, the two nested combinations in the tuple
are merged using $\mathit{\mathit{KcombsAlg}}\left(K\right):\left(\mathit{NCss},\mathit{NCss}\right)\to\mathit{NCss}$.

\subsection{Deep incremental cell enumeration (Deep-ICE) algorithm and symmetry
	fusion}

As noted, working with
the maxout network enables the application of an additional fusion
principle—an extension of the symmetric fusion theorem proposed by
\citet{he2023efficient} for linear classification.
\begin{thm}
Symmetric fusion for maxout neuron.\emph{ Given a maxout neuron
		defined by $K$ hyperplane. If the predictions associated with these $K$ hyperplanes are known, then the predictions
		for the configuration obtained by reversing the direction of all normal
		vectors can be obtained directly.}
\end{thm}
\begin{proof}
	See appendix \ref{subsec:Symmetric-fusion}.
\end{proof}
The symmetric fusion theorem eliminates half of the computation, allowing
us to enumerate all $2^{K}$ possible orientations of hyperplanes
using only $2^{K-1}$ of them. Consequently, the problem \ref{Problem specification}
can be reformulated more efficiently by applying the symmetric fusion

\[
	\protect\begin{aligned}\mathit{DeepICE}\left(D,K\right)= \mathit{min}_{\text{0-1}}\left(K\right)\circ\mathit{eval}^{\prime}\left(K-1\right)\circ\mathit{\mathit{nestedCombs}}\left(D,K\right),
		\protect\end{aligned}
\]
where $\mathit{eval}^{\prime}\left(K-1\right) = \mathit{eval}\left(K\right)\circ\mathit{cp}\left(\mathit{basgns}\left(K-1\right)\right) $.

We are now ready to derive the Deep-ICE algorithm, which follows as
a direct consequence of the following lemma.
\begin{lem}
	\emph{Let $\mathit{DeepICEAlg}$ be defined as
		\begin{equation}
			\mathit{DeepICEAlg}\left(D,K\right)=\mathit{min}_{\text{0-1}}\left(K\right)\circ\mathit{eval}^{\prime}\left(K-1\right)\circ\mathit{\mathit{nestedCombsAlg}}\left(D,K\right),\label{eq: deepice-definition}
		\end{equation}
		where $\mathit{eval}^{\prime}\left(K-1\right)$ evaluates $E_{\text{0-1}}$
		for each nested combination returned by $\mathit{\mathit{nestedCombsAlg}}\left(D,K\right)$.
		Then the following fusion condition holds:}
	
	\emph{
		\begin{equation}
			\mathit{DeepICE}\left(D,K\right)=f\circ\mathit{nestedCombsAlg}\left(D,K\right)=\mathit{DeepICEAlg}\left(D,K\right)\circ f\times f,
		\end{equation}
		where $f=\mathit{min}_{\text{0-1}}\left(D\right)\circ\mathit{eval}^{\prime}\left(K-1\right)$,
		which defines Algorithm (\ref{alg:Deep-ICE-algorithm}).}
\end{lem}
See Appendix \ref{subsec:Proof-of-fusion} for detailed proof. Algorithm (\ref{eq: deepice-definition}) has a worst-case complexity of $O\left(N^{DK+1}\right)$, which is formally established in the following theorem.
\begin{thm}
	\emph{The DeepICE algorithm has a time complexity of $O\left(K\times N\times2^{K-1}\times\left(\begin{array}{c}
			\left(\begin{array}{c}
				N\\
				D
			\end{array}\right)\\
			K
		\end{array}\right)+N\times D^{3}\times\left(\begin{array}{c}
			N\\
			D
		\end{array}\right)\right)$ which is strictly smaller than $O\left(N^{DK+1}\right)$, and a space
		complexity of $O\left(\left(\begin{array}{c}
			\left(\begin{array}{c}
				N\\
				D
			\end{array}\right)\\
			K-1
		\end{array}\right)\times K+\left(\begin{array}{c}
			N\\
			D-1
		\end{array}\right)\times N\right)$, which is strictly smaller than $O\left(N^{D\left(K-1\right)}\right)$.}
\end{thm}
See Appendix \ref{subsec:Complexity-analysis} for detailed proof.

In practice, we provide two implementations for (\ref{eq: deepice-definition}) (see \ref{subsec: algorithms}). The sequential version enables two techniques that substantially improve memory efficiency and runtime performance. The D\&C version, which builds upon the sequential definition, supports embarrassingly parallel execution.

Figure \ref{fig:run-time-polynomial}  show that the
empirical wall-clock runtime of our algorithm aligns with our worst-case
complexity analysis.
\paragraph{Generalization to deep neural networks}

Our algorithm generalizes naturally to deep neural networks. Deeper
networks can be viewed as compositions of hidden neurons from preceding
layers, where linear combinations of these neurons form the predictions
of the subsequent layer. Hence, each layer is essentially a function
of the predictions generated in the layer before it. Suppose the $i$-th
hidden layer contains $K_{i}$ hidden nodes. Computing all possible
predictions for this layer has complexity $O\left(N^{D\times K_{1}\times{K_{2}}\times{K_{3}}\ldots\times{K_{i}}}\right)$.
For instance, the optimal solution of a three-layer network is a nested-nested
combination, while a four-layer network corresponds to a nested-nested-nested
combination. Solving a three-layer network requires complexity $O\left(N^{D\times K_{1}\times K_{2}}\right)$.
Consequently, obtaining exact solutions for deeper networks is practically
infeasible due to combinatorial explosion.

One way to mitigate this challenge is to train a deep network greedily,
where the computation of the second hidden layer depends only on the
first. In this case, the complexity becomes $O\left(N^{D\times K_{1}}+K_{1}^{K_{2}}+K_{2}^{K_{3}}\ldots+K_{i-1}^{K_{i}}\right)$ for network with $i$ layers.
Under this scheme, regardless of depth, the overall complexity is
dominated by that of the first hidden layer.

\section{Empirical analysis}

We evaluate the performance of our Deep-ICE algorithm against two
baselines: support vector machines (SVMs) and an identical neural
network architecture trained using Adam algorithm, referred to as MLP. The MLP is optimized with binary
cross-entropy loss with logits, using the entire training dataset
as a single batch in each epoch. The evaluation is conducted across
11 datasets from the UCI Machine Learning Repository. Since we assume
data are in general position, which requires affine independence of
the data, we remove duplicate entries and add a zero mean Gaussian
noise (standard deviation $1\times10^{-8}$, small enough that
it does not affect the results of SVM and MLP) to each dataset. All experiments were conducted on a single GeForce RTX 4060 Ti GPU.

\begin{figure}[h]
	\begin{subfigure}[b]{0.45\textwidth}
		\centering
		\includegraphics[scale=0.18, trim=0 40 0 0, clip]{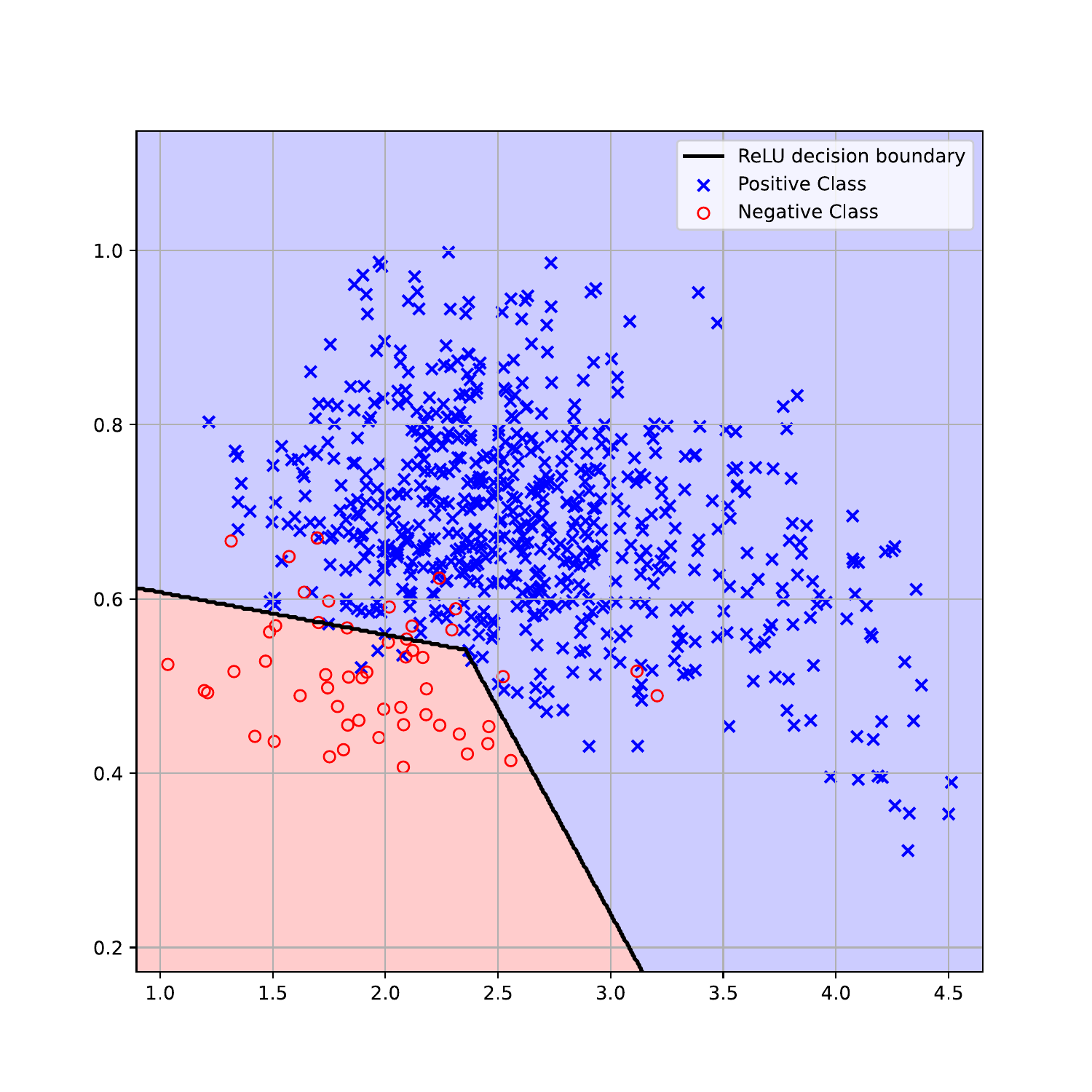}
		\caption{Deep-ICE (0-1 loss: 16)}
		\label{fig:voicemap-left}
	\end{subfigure}
	\hfill
	\begin{subfigure}[b]{0.45\textwidth}
		\centering
		\includegraphics[scale=0.18, trim=0 40 0 0, clip]{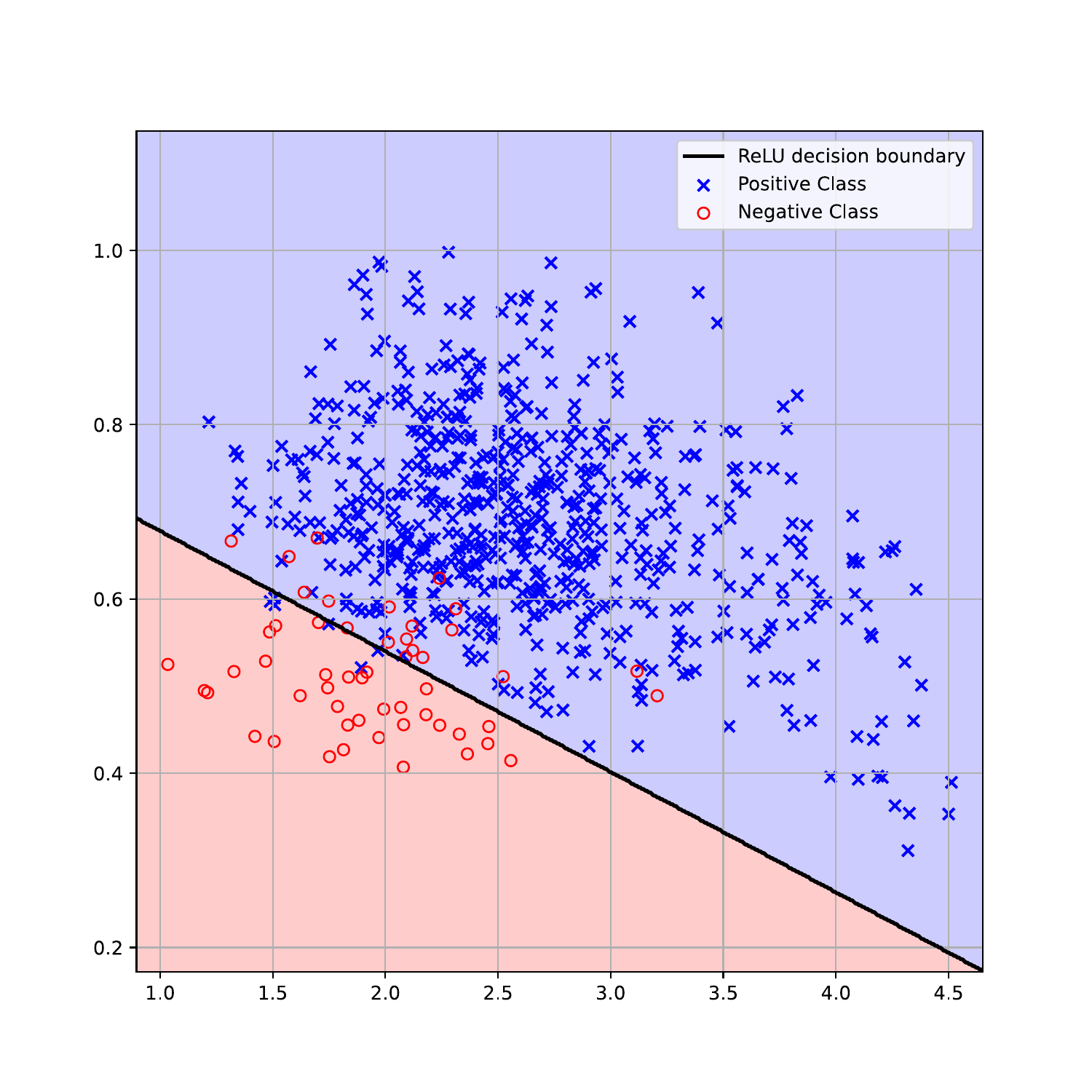}
		\caption{Gradient descent (0-1 loss: 25)}
		\label{fig:voicemap-right}
	\end{subfigure}
	\caption{The global optimal solution of a rank-2 maxout network with one neuron
		on a real-world dataset containing $N=704$ data items in $\mathbb{R}^{2}$.}
	\label{fig:voicemap}
\end{figure}

\paragraph{Exact solution vs. gradient descent}

Figure \ref{fig:voicemap} illustrates the ERM solution and the gradient
descent outcome for a rank-2 maxout network with one maxout neuron.
Previously, \citet{Xi_Exact_0-1_loss_2023} reported 0-1 losses of
19 and 23 for the global optimal linear model and the SVM, respectively,
on this dataset. In contrast, ERM solution obtained by $\mathit{DeepICE}$,
achieves only 16 misclassifications, compared to 25 for the same architecture
trained via gradient descent. Notably, despite a rank-2 maxout neuron
involves two hyperplanes, the gradient-based solution uses only one;
the second hyperplane lies outside the data region and does not contribute
to predictions.

\paragraph{Exact solution over coresets}

\begin{table}
	\scriptsize
	\caption{Five-fold cross-validation results on the UCI dataset. We compare the performance of our Deep-ICE algorithm—trained either with the coreset selection method or directly by Deep-ICE algorithm (marked by *)—against approximate methods: SVM and a maxout network trained via gradient descent (denoted as MLP). Results are reported as mean 0–1 loss over training and test sets in the format: Training Error / Test Error (Standard Deviation: Train / Test). The best-performing algorithm in each row is highlighted in bold. 
		 \label{tab:5-fold-cross-validation experiments}}
	
	\begin{center} 
		\begin{tabular}{@{} 
				>{\raggedright}p{0.03\textwidth}  
				>{\raggedleft}p{0.03\textwidth} 
				>{\raggedleft}p{0.02\textwidth} 
				>{\raggedleft}p{0.1\textwidth} 
				>{\raggedleft}p{0.1\textwidth} 
				>{\raggedleft}p{0.1\textwidth} 
				>{\raggedleft}p{0.09\textwidth} 
				>{\raggedleft}p{0.09\textwidth} 
				>{\raggedleft}p{0.08\textwidth} 
				>{\raggedleft}p{0.1\textwidth} 
				@{}}
			
			Dataset & $N$ & $D$ & Deep-ICE (\%) ($K=1$) &  Deep-ICE (\%) ($K=2$) &  Deep-ICE (\%) ($K=3$) &  SVM (\%) & MLP (\%) ($K=1$) &  MLP (\%) ($K=2$)  &  MLP (\%) ($K=3$)  			\tabularnewline
			\midrule  
\addlinespace[0.5cm]

		 Ai4i & 10000 & 6 & 
			
97.45/97.40

(0.10/0.36)
&
\textbf{97.90}/\textbf{97.82}

(0.01/0.35)
&
97.71/97.71

(0.10/0.25)
&
96.64/96.48

(0.11/0.44)
&
97.01/96.90

(0.11/0.40)
&
97.20/97.02

(0.18/0.39)
&
97.56/97.55

(0.13/0.46)
			\tabularnewline

			Caesr & 72 & 5 & 
*74.55/82.67

(7.18/16.11)
&
\textbf{89.45}/\textbf{88.00}

(4.21/9.80)
&
84.36/86.67

(7.51/5.96)
&
72.00/57.33

(7.14/6.80)
&
71.64/62.67

(6.76/6.80)
&
76.36/56.00

(6.19/9.04)
&
81.82/60.00

(1.15/11.93)
			\tabularnewline

			VP & 704 & 2 & 
*96.94/\textbf{97.59}

(0.44/1.46)
&
97.76/\textbf{97.59}

(0.41/1.65)
&
\textbf{97.80}/97.45

(0.43/1.71)
&
96.77/97.02

(0.44/2.07)
&
96.63/96.74

(0.50/2.13)
&
96.77/97.02

(0.73/1.64)
&
96.63/96.74

(0.50/2.13)
			\tabularnewline

			Spesis & 975 & 3 & 
*94.47/92.88

(0.10/0.61)
&
\textbf{96.43}/95.26

(0.49/1.82)
&
96.24/\textbf{95.36}

(0.22/1.62)
&
94.46/92.43

(0.10/0.38)
&
94.46/92.43

(0.10/0.38)
&
94.46/92.55

(0.10/0.51)
&
94.46/92.43

(0.10/0.38)
			\tabularnewline
			 
			HB & 283 & 3 & 
*77.18/75.44

(0.45/2.48)
&
80.11/77.19

(0.74/2.48)
&
\textbf{80.85}/\textbf{78.53}

(1.02/3.57)
&
72.40/71.23

(0.46/2.38)
&
72.82/74.80

(0.66/2.08)
&
75.34/75.26

(0.86/2.51)
&
75.97/73.92

(0.18/2.08)
			\tabularnewline
			 
			BT & 502 & 4 & 
		*77.13/76.36
		
		(1.46/2.71)
&
		\textbf{79.59}/\textbf{77.98}
		
		(0.62/3.38)
&
		79.36/\textbf{77.98}
		
		(0.59/2.88)
&
		75.09/70.14
		
		(0.51/0.76)
&
		76.17/73.54
		
		(1.05/3.64)
&
		76.11/73.54
		
		(1.01/2.06)
&
		76.29/75.45
		
		(1.02/2.29)
			\tabularnewline
			 
			AV & 2342 & 7 & 
89.89/88.52

(0.33/1.56)
&
\textbf{90.34}/\textbf{89.04}

(0.15/1.39)
&
89.77/88.76

(0.33/1.75)
&
87.16/87.26

(0.31/1.24)
&
86.92/87.20

(0.24/0.71)
&
87.18/86.88

(0.24/0.66)
&
87.63/87.31

(0.44/0.73)
			\tabularnewline
			 
			SO & 1941 & 27 & 
\textbf{77.77}/\textbf{76.03}

(0.43/0.83)
&
77.13/75.33

(0.81/1.32)
&
76.66/74.95

(0.74/1.38)
&
73.67/70.80

(0.52/2.05)
&
74.81/72.13

(0.44/1.63)
&
77.09/71.71

(0.26/1.66)
&
78.33/74.68

(0.40/2.31)
			\tabularnewline
			 
			DB & 1146 & 9 & 
78.78/79.69

(0.41/0.69)
&
83.60/\textbf{81.37}

(0.43/2.52)
&
\textbf{83.88}/81.32

(0.98/2.23)
&
69.72/67.62

(0.65/2.86)
&
76.13/74.77)

(0.41/2.01
&
77.64/76.19

(0.65/1.06)
&
77.85/75.11

(0.89/0.72)
			\tabularnewline
			 
			RC & 3810 & 7 & 
93.88/92.45

(0.28/1.02)
&
93.91/\textbf{93.10}

(0.24/1.02)
&
\textbf{93.94}/92.98

(0.21/0.98)
&
93.05/91.75

(0.25/1.12)
&
93.30/92.10

(0.28/1.07)
&
93.30/92.15

(0.30/1.15)
&
93.30/92.12

(0.29/1.13)
			\tabularnewline
			 
			SS & 51433 & 3 & 
86.57/\textbf{86.72}

(0.03/0.15)
&
\textbf{86.60}/\textbf{86.72}

(0.04/0.16)
&
86.59/86.70

(0.03/0.11)
&
82.77/82.75

(0.06/0.22)
&
79.73/79.73

(0.15/0.20)
&
79.65/79.65

(0.18/0.16)
&
79.48/79.73

(0.07/0.04)
			\tabularnewline

		\end{tabular}
		\par\end{center}
	
\end{table}

Exact solutions typically require an exhaustive exploration of the
configuration space. Achieving exact optimality on training data is
often unnecessary, as such solutions may not generalize well to out-of-sample
data.

Instead, generating multiple high-quality candidate solutions enables
selection based on validation or test performance. For example, SVMs
provide tunable hyperparameters to generate alternative models, while
gradient-based MLPs yield varied solutions via different random seed
initializations. However, both approaches require computationally
expensive retraining to explore alternatives, often without principled
guidance. Attempts to automate this process frequently rely on strong
probabilistic assumptions that rarely hold in practice \citep{shahriari2015taking,klein2017fast}
or employ empirical heuristics \citep{liao2022empirical,wainer2021tune,duan2003evaluation},
resulting in substantial computational waste due to redundant retraining.

A common approach to address this issue in studies of exact algorithms
is to use multiple random initializations. However, this approach
often becomes ineffective as data scales increase. Each run typically
uses a manually set time limit, but this still results in redundant
retraining. To address these challenges, we propose a coreset selection
method, detailed in Algorithm \ref{coreset-selection-method}. Instead
of computing the exact solution across the entire dataset, which is
computationally infeasible for large $K$ and $D$, our approach identifies
the exact solution for the most representative subsets. By shuffling
the data, the input will unlikely be the ordering that is pathological
i.e., one where the optimal solution is obtained only at a late stage
of the recursive process in the Deep-ICE algorithm. This method can
effectively explore thousands of candidate configurations in the coresets
that have lower training accuracy than SVMs and MLPs. In our experiments,
we trained a two-layer maxout network using the algorithmic process
described in \ref{coreset-selection-method}. In 5-fold cross-validation
tests, our method demonstrated significantly better performance. These results consistently
outperformed those of SVMs and the same maxout network trained with
gradient descent.

Due to the ability to generate an extensive number of candidate solutions,
we observed several interesting findings in our experiments. Although
extensive prior research suggests that the maximal-margin (MM) classifier
(i.e., SVM) offers theoretical guarantees for test accuracy \citep{mohri2012foundations},
we found that the MM classifier does not always perform as expected.
Specifically, we did not find clear evidence that the MM classifier
consistently achieves better out-of-sample performance. A more detailed
analysis is provided in Appendix \ref{subsec:Experiments-of-exhuastively}.

Furthermore, \citet{karpukhin2024exact} proposed an interesting framework
that introduces stochasticity into the model's output and optimizes
the expected accuracy, allowing gradient-based methods to directly
optimize accuracy rather than surrogate losses. However, despite being
named EXACT, the method is actually short for ``EXpected ACcuracy
opTimization'' and is therefore a stochastic approach rather than
a deterministic exact algorithm. We include a comparison with their
framework in Appendix \ref{subsec:Additional-experiments}, which
shows that it outperforms MLPs trained with surrogate losses.

Additionally, the wall-clock runtime comparison between EXACT and
MLP is provided in \ref{subsec:Additional-experiments}.

\section{Discussion and conclusion}

In this paper, we present the first algorithm for finding the globally minimal empirical risk of two-layer neural networks under 0–1 loss. The algorithm achieves polynomial time and space complexity for fixed $D$ and $K$. The DeepICE algorithm is specifically designed to optimize both efficiency and parallelizability. Even without bounding techniques to accelerate computation, our implementation demonstrates strong performance: it can handle over $1\times10^{11}$ configurations within minutes, highlighting the intrinsic efficiency of our algorithm independent of any bounding methods. Incorporating additional bounding techniques in future research could further enhance its scalability.

Another key contribution of this paper is the empirical evidence that optimal solutions do not necessarily overfit the data. Our out-of-sample tests indicate that solutions trained using our method, which achieve significantly higher training accuracy than SVMs or two-layer neural networks, still perform well on unseen data when model complexity is properly controlled. This finding points to a promising avenue for applying our algorithm to problems where both interpretability and model complexity are critical.

\section*{Reproducibility Statement}
	To facilitate reproducibility, we provide \textbf{three} versions of our algorithm: a \textit{recursive version}, a \textit{divide-and-conquer version}, and a \textit{sequential definition} in Appendix \ref{subsec: algorithms}. The recursive version is written clearly in a functional style and can be executed in a functional programming language with minimal syntactic adjustments, allowing the algorithm to run with no ambiguity. In addition, imperative implementations in both Python and CUDA are included in supplementary materials, along with all datasets used in our experiments. Enabling independent verification and replication of the results reported in this paper.

\bibliography{deepice_iclr26}

@book{bird2020algorithm,
  title={Algorithm Design with Haskell},
  author={Bird, Richard and Gibbons, Jeremy},
  year={2020},
  publisher={Cambridge University Press}
}

@article{arora2016understanding,
  title={Understanding deep neural networks with rectified linear units},
  author={Arora, Raman and Basu, Amitabh and Mianjy, Poorya and Mukherjee, Anirbit},
  journal={ArXiv preprint ArXiv:1611.01491},
  year={2016}
}

@book{hertrich2022facets,
  title={Facets of neural network complexity},
  author={Hertrich, Christoph},
  year={2022},
  publisher={Technische Universitaet Berlin (Germany)}
}

@article{goel2020tight,
  title={Tight hardness results for training depth-2 ReLU networks},
  author={Goel, Surbhi and Klivans, Adam and Manurangsi, Pasin and Reichman, Daniel},
  journal={ArXiv preprint ArXiv:2011.13550},
  year={2020}
}

@article{bartlett2019nearly,
  title={Nearly-tight VC-dimension and pseudodimension bounds for piecewise linear neural networks},
  author={Bartlett, Peter L and Harvey, Nick and Liaw, Christopher and Mehrabian, Abbas},
  journal={Journal of Machine Learning Research},
  volume={20},
  number={63},
  pages={1--17},
  year={2019}
}

@article{bird1996algebra,
  title={The algebra of programming},
  author={Bird, Richard and De Moor, Oege},
  journal={NATO ASI DPD},
  volume={152},
  pages={167--203},
  year={1996}
}

@book{mohri2012foundations,
  title={Foundations of machine learning},
  author={Mohri, Mehryar and Rostamizadeh, Afshin and Talwalkar, Ameet},
  year={2012},
  publisher={MIT press}
}

@article{he2023efficient,
  title={An efficient, provably exact algorithm for the 0-1 loss linear classification problem},
  author={He, Xi and Little, Max A},
  journal={ArXiv preprint ArXiv:2306.12344},
  year={2023}
}

@article{he2024ekm,
  title={{EKM}: an exact, polynomial-time algorithm for the $ K $-medoids problem},
  author={He, Xi and Little, Max A},
  journal={ArXiv preprint ArXiv:2405.12237},
  year={2024}
}

@article{froese2022computational,
  title={The computational complexity of ReLU network training parameterized by data dimensionality},
  author={Froese, Vincent and Hertrich, Christoph and Niedermeier, Rolf},
  journal={Journal of Artificial Intelligence Research},
  volume={74},
  pages={1775--1790},
  year={2022}
}

@software{Xi_Exact_0-1_loss_2023,
author = {Xi, He and Little, Max A.},
doi = {10.5281/zenodo.7814259},
month = apr,
title = {{Exact 0-1 loss linear classification algorithms}},
url = {https://github.com/XiHegrt/E01Loss},
version = {1.0.1},
year = {2023}
}

@article{kirchner2016machine,
  title={Machine Bias: There’s software used across the country to predict future criminals. And it’s biased against blacks},
  author={Kirchner, Julia and Angwin, Surya and Mattu, Jeff and Larson, Lauren},
  journal={Pro Publica: New York, NY, USA},
  year={2016}
}

@article{mcgough2018bad,
  title={How bad is Sacramento’s air, exactly? Google results appear at odds with reality, some say},
  author={McGough, Michael},
  journal={Sacramento Bee},
  volume={7},
  year={2018}
}

@inproceedings{klein2017fast,
  title={Fast bayesian optimization of machine learning hyperparameters on large datasets},
  author={Klein, Aaron and Falkner, Stefan and Bartels, Simon and Hennig, Philipp and Hutter, Frank},
  booktitle={Artificial intelligence and statistics},
  pages={528--536},
  year={2017},
  organization={PMLR}
}

@article{shahriari2015taking,
  title={Taking the human out of the loop: A review of Bayesian optimization},
  author={Shahriari, Bobak and Swersky, Kevin and Wang, Ziyu and Adams, Ryan P and De Freitas, Nando},
  journal={Proceedings of the IEEE},
  volume={104},
  number={1},
  pages={148--175},
  year={2015},
  publisher={IEEE}
}

@article{duan2003evaluation,
  title={Evaluation of simple performance measures for tuning SVM hyperparameters},
  author={Duan, Kaibo and Keerthi, S Sathiya and Poo, Aun Neow},
  journal={Neurocomputing},
  volume={51},
  pages={41--59},
  year={2003},
  publisher={Elsevier}
}

@article{wainer2021tune,
  title={How to tune the RBF SVM hyperparameters? An empirical evaluation of 18 search algorithms},
  author={Wainer, Jacques and Fonseca, Pablo},
  journal={Artificial Intelligence Review},
  volume={54},
  number={6},
  pages={4771--4797},
  year={2021},
  publisher={Springer}
}

@article{liao2022empirical,
  title={An empirical study of the impact of hyperparameter tuning and model optimization on the performance properties of deep neural networks},
  author={Liao, Lizhi and Li, Heng and Shang, Weiyi and Ma, Lei},
  journal={ACM Transactions on Software Engineering and Methodology (TOSEM)},
  volume={31},
  number={3},
  pages={1--40},
  year={2022},
  publisher={ACM New York, NY}
}

@inproceedings{kolmogorov1957representations,
  title={On the representations of continuous functions of many variables by superposition of continuous functions of one variable and addition},
  author={Kolmogorov, Andrei Nikolaevich},
  booktitle={Dokl. Akad. Nauk USSR},
  volume={114},
  pages={953--956},
  year={1957}
}

@article{rudin2019stop,
  title={Stop explaining black box machine learning models for high stakes decisions and use interpretable models instead},
  author={Rudin, Cynthia},
  journal={Nature machine intelligence},
  volume={1},
  number={5},
  pages={206--215},
  year={2019},
  publisher={Nature Publishing Group UK London}
}

@article{karpukhin2024exact,
  title={EXACT: How to train your accuracy},
  author={Karpukhin, Ivan and Dereka, Stanislav and Kolesnikov, Sergey},
  journal={Pattern Recognition Letters},
  volume={185},
  pages={23--30},
  year={2024},
  publisher={Elsevier}
}

@article{bai2023efficient,
  title={Efficient global optimization of two-layer relu networks: Quadratic-time algorithms and adversarial training},
  author={Bai, Yatong and Gautam, Tanmay and Sojoudi, Somayeh},
  journal={SIAM Journal on Mathematics of Data Science},
  volume={5},
  number={2},
  pages={446--474},
  year={2023},
  publisher={SIAM}
}
\bibliographystyle{iclr2026_conference}

\appendix

\section{Proofs}

\subsection{Symmetric fusion for maxout network\label{subsec:Symmetric-fusion}}
\begin{thm}
	Symmetric fusion for maxout network.\emph{ Given a maxout network
		defined by $K$ hyperplane (neurons).If the predictions associated
		with this configuration of $K$ hyperplanes are known, then the predictions
		for the configuration obtained by reversing the direction of all normal
		vectors can be obtained directly from the original hyperplanes, without
		explicitly recomputing the predictions for the reversed hyperplanes.}
\end{thm}
\begin{proof}
	Consider a maxout network defined by $K$ hyperplanes $\mathcal{H}=\left\{ h_{k}\mid k\in\mathcal{K}=\left\{ 1,2,\ldots,K\right\} \right\} $,
	where each hyperplane $h_{k}$ is defined by a normal vector $\boldsymbol{w}_{k}:\mathbb{R}^{D}$.
	Together these hyperplanes define a decision function $f_{\boldsymbol{W}_{1},\boldsymbol{W}_{2}}\left(\boldsymbol{x}\right)$.
	Equation (\ref{Problem specification}) implies that a data item $\boldsymbol{x}$
	is predicted to negative class by $f_{\boldsymbol{W}_{1},\boldsymbol{W}_{2}}\left(\boldsymbol{x}\right)$
	if and only it lies in the negative sides of all hyperplanes in $\mathcal{H}$,
	because $f_{\boldsymbol{W}_{1},\boldsymbol{W}_{2}}\left(\boldsymbol{x}\right)$
	will return positive as long as there exists a $k$ such that $\boldsymbol{w}_{k}\boldsymbol{x}\geq0$.
	Therefore, the prediction labels of the two-layer NN $\boldsymbol{y}_{\text{maxout}}$
	consists of the union of positive prediction labels for each hyperplane
	$h_{k}$, and the remaining data item, which lies in the negative
	side with respect to all $K$ hyperplanes will be assigned to negative
	class. class. In other words, if we denote $\boldsymbol{y}^{+}$ and
	$\boldsymbol{y}^{-}$ as the positive and negative prediction indexes
	of $\boldsymbol{y}$ respectively, then we have
	\begin{equation}
		\begin{aligned}\boldsymbol{y}_{\text{maxout}}^{+} & =\bigcup_{k\in\mathcal{K}}\boldsymbol{y}_{k}^{+}\\
			\boldsymbol{y}_{\text{maxout}}^{-} & =\mathcal{D}\backslash\boldsymbol{y}_{\text{maxout}}^{+}
		\end{aligned}
	\end{equation}
	where $\backslash$ is defined as the set difference and $\bigcup_{k\in\mathcal{K}}\boldsymbol{y}_{k}^{+}$
	denote the union of $\boldsymbol{y}_{k}^{+}$, $k\in\mathcal{K}$.
	For instance, if $\boldsymbol{y}_{1}=\left(1,1,-1,-1\right)$ and
	$\boldsymbol{y}_{2}=\left(-1,1,1,-1\right)$, then $\boldsymbol{y}_{1}^{+}=\left\{ 1,2\right\} $
	and $\boldsymbol{y}_{2}^{+}=\left\{ 2,3\right\} $, thus $\boldsymbol{y}_{1}^{+}\cup\boldsymbol{y}_{2}^{+}=\left\{ 1,2,3\right\} $
	
	For a two-layer maxout NN, the data points can be classified into
	three categories based on their relationship to the $K$ hyperplanes
	defined by the $K$ hidden neurons:
	
	1. Data points that lie in the region where all $K$ hyperplanes are
	on the positive side.
	
	2. Data points that lie in the region where all $K$ hyperplanes are
	on the negative side.
	
	3. Data points that lie in the region where some hyperplanes are on
	the positive side and others are on the negative side.
	
	If we reverse the orientation of all $K$ hyperplanes in $\mathcal{H}$,
	i.e., $\boldsymbol{w}_{k}=-\boldsymbol{w}_{k}$. Only data points
	that fall into the class of the first two cases will be reversed,
	because the prediction labels of these data be reversed if the orientation
	for all hyperplanes is reversed, the classification of data points
	in the third category will remain unchanged. This is because (\ref{eq: relu decision funciton})
	implies that, the prediction labels of the two-layer NN, $\boldsymbol{y}_{\text{maxout}}$,
	consist of the union of positive prediction labels for each hyperplane
	$h_{k}$.Therefore, reversing the direction of all hyperplanes will
	affect only data points $\boldsymbol{x}_{n}$ that lie in the positive
	class for all hyperplanes, ($n\in\boldsymbol{y}_{k}^{+}$, $\forall k\in\mathcal{K}$)
	or the negative class for all hyperplanes ($n\in\boldsymbol{y}_{k}^{-}$,
	$\forall k\in\mathcal{K}$ ) will be change the label. For any other
	data points, there always exists at least one hyperplane that classifies
	them as negative. After reversing the direction of all hyperplanes,
	this same hyperplane will classify these points as positive, leaving
	their prediction labels unchanged.
\end{proof}

\subsection{Proof of nested combination generator\label{subsec: Proof of nested generator}}

Given $\mathit{nestedCombsAlg}\left(D,K\right)$ defined as

\begin{align}
	\Big\langle & \mathit{setEmpty}(D) \circ \mathit{KcombsAlg}(K) \circ \mathit{Ffst},\quad \nonumber \\
	& \mathit{KcombsAlg}(K) \circ \Big\langle 
	\mathit{Kcombs}(K) \circ !!\left(D\right) \circ \mathit{KcombsAlg}(D) \circ \mathit{Ffst},\,
	\mathit{KcombsAlg}(K) \circ \mathit{Fsnd}
	\Big\rangle 
	\Big\rangle, \label{nested combs gen-abstract}
\end{align}

We need to verify the following fusion condition

\begin{equation}
	f\circ\mathit{KcombsAlg}\left(D\right)=\mathit{nestedCombsAlg}\left(D,K\right)\circ f\times f,
\end{equation}
where $f=\left\langle \mathit{setEmpty}\left(D\right),\mathit{Kcombs}\left(K\right)\circ!!\left(D\right)\right\rangle $.
In other words, we need to prove that the following diagram commutes

\[
\xymatrix{\mathit{Css}\ar[d]_{f} &  &  &  & \left(\mathit{Css},\mathit{Css}\right)\ar[d]^{f\times f}\ar[llll]_{\mathit{\mathit{kcombsAlg}\left(D\right)}}\\
	\left(\mathit{Css},\mathit{NCss}\right) &  &  &  & \left(\left(\mathit{Css},\mathit{NCss}\right),\left(\mathit{Css},\mathit{NCss}\right)\right)\ar[llll]^{\mathit{nestedCombsAlg}\left(D,K\right)}
}
\]

However, proving that the above diagram commutes is challenging. Instead,
we expand the diagram by presenting all intermediate stage explicitly

\subsubsection*{
	\[
	\protect\xymatrix{\mathit{Css}\ar[d]^{\left\langle \mathit{SE}\left(D\right),!!\left(D\right)\right\rangle } &  &  &  &  &  &  & \left(\mathit{Css},\mathit{Css}\right)\ar[d]_{\left\langle \mathit{SE}\left(D\right),!!\left(D\right)\right\rangle \times\left\langle \mathit{SE}\left(D\right),!!\left(D\right)\right\rangle }\ar[lllllll]^{\mathit{KCsA}\left(D\right)}\\
		\left(\mathit{Css},\mathit{Cs}\right)\ar[d]^{\mathit{id}\times\mathit{KCs}\left(K\right)} &  & \left(\mathit{Css},\left(\mathit{Cs},\mathit{Cs}\right)\right)\ar[ll]^{\mathit{SE}\left(D\right)\times\cup} &  &  &  &  & \left(\left(\mathit{Css},\mathit{Cs}\right),\left(\mathit{Css},\mathit{Cs}\right)\right)\ar[d]_{\left(\mathit{id}\times\mathit{KCs}\left(K\right)\right)\times\left(\mathit{id}\times\mathit{KCs}\left(K\right)\right)}\ar[lllll]^{\texttt{\texttt{\ensuremath{\left\langle \mathit{KCsA}\left(D\right)\circ\mathit{Ffst},\left\langle !!\left(D\right)\circ\mathit{KCsA}\left(D\right)\circ\mathit{Ffst},\cup\circ\mathit{Fsnd}\right\rangle \right\rangle }}}}\\
		\left(\mathit{Css},\mathit{NCss}\right) &  & \left(\mathit{Css},\left(\mathit{NCss},\mathit{NCss}\right)\right)\ar[ll]^{\mathit{SE}\left(D\right)\times\mathit{KcsA}\left(K\right)\circ\cup} &  &  &  &  & \left(\left(\mathit{Css},\mathit{NCss}\right),\left(\mathit{Css},\mathit{NCss}\right)\right)\ar[lllll]^{\texttt{\ensuremath{\left\langle \mathit{KCsA}\left(D\right)\circ\mathit{Ffst},\left\langle \mathit{KCs}\left(K\right)\circ!!\left(D\right)\circ\mathit{KCsA}\left(D\right)\circ\mathit{Ffst},\mathit{KCsA}\left(K\right)\circ\mathit{Fsnd}\right\rangle \right\rangle }}}
	}
	\]
}

where $\cup\left(a,b\right)=a\cup b$, and $\mathit{\mathit{SE}}$,
$\mathit{KCs}$ and $\mathit{KCsA}$ are short for $\mathit{setEmpty}$,
$\mathit{Kcombs}$ and $\mathit{KcombsAlg}$.

To prove the fusion condition, we first need to verify the two paths
between $\left(\mathit{Css},\mathit{Css}\right)$ and $\left(\mathit{Css},\mathit{Cs}\right)$.
In other words, we need to prove
\begin{equation}
	\begin{aligned} & \left\langle \mathit{SE}\left(D\right),!!\left(D\right)\right\rangle \circ\mathit{KCsA}\left(D\right)=\\
		& \qquad SE\left(D\right)\times\left(\cup\circ\texttt{\texttt{\ensuremath{\left\langle \mathit{KCsA}\left(D\right)\circ\mathit{Ffst},\left\langle !!\left(D\right)\circ\mathit{KCsA}\left(D\right)\circ\mathit{Ffst},\cup\circ\mathit{Fsnd}\right\rangle \right\rangle }}}\right)\circ\left(\left\langle \mathit{SE}\left(D\right),!!\left(D\right)\right\rangle \times\left\langle \mathit{SE}\left(D\right),!!\left(D\right)\right\rangle \right)
	\end{aligned}
\end{equation}

This can be proved by following equational reasoning

\begin{align*}
	& \mathit{SE}\left(D\right)\times\cup\circ\texttt{\texttt{\ensuremath{\left\langle \mathit{KCsA}\left(D\right)\circ\mathit{Ffst},\left\langle !!\left(D\right)\circ\mathit{KCsA}\left(D\right)\circ\mathit{Ffst},\cup\circ\mathit{Fsnd}\right\rangle \right\rangle }}}\circ\left(\left\langle SE\left(D\right),!!\left(D\right)\right\rangle \times\left\langle \mathit{SE}\left(D\right),!!\left(D\right)\right\rangle \right)\\
	\equiv & \text{ \ensuremath{\times} absorption law}\\
	& \texttt{\texttt{\ensuremath{\left\langle \mathit{SE}\left(D\right)\circ\mathit{KCsA}\left(D\right)\circ\mathit{Ffst},\cup\circ\left\langle !!\left(D\right)\circ\mathit{KCsA}\left(D\right)\circ\mathit{Ffst},\cup\circ\mathit{Fsnd}\right\rangle \right\rangle }}}\circ\left(\left\langle SE\left(D\right),!!\left(D\right)\right\rangle \times\left\langle \mathit{SE}\left(D\right),!!\left(D\right)\right\rangle \right)\\
	\equiv & \text{ Product fusion}\\
	& \big<\mathit{SE}\left(D\right)\circ\mathit{KCsA}\left(D\right)\circ\mathit{Ffst}\circ\left(\left\langle SE\left(D\right),!!\left(D\right)\right\rangle \times\left\langle \mathit{SE}\left(D\right),!!\left(D\right)\right\rangle \right),\\
	& \qquad\quad\texttt{\texttt{\ensuremath{\cup\circ\left\langle !!\left(D\right)\circ\mathit{KCsA}\left(D\right)\circ\mathit{Ffst},\cup\circ\mathit{Fsnd}\right\rangle \circ\left(\left\langle SE\left(D\right),!!\left(D\right)\right\rangle \times\left\langle \mathit{SE}\left(D\right),!!\left(D\right)\right\rangle \right)\big>}}}\\
	\equiv & \text{ Definition of \ensuremath{\mathit{SE\left(D\right)}} and product fusion}\\
	& \texttt{\texttt{\ensuremath{\left\langle \mathit{SE}\left(D\right)\circ\mathit{KCsA}\left(D\right),\cup\circ\left\langle !!\left(D\right)\circ\mathit{KCsA}\left(D\right)\circ\mathit{Fse}\left(D\right),\cup\circ\mathit{F!!}\left(D\right)\right\rangle \right\rangle }}}\\
	\equiv & \text{ Definition of Combination}\\
	& \left\langle \mathit{SE}\left(D\right),!!\left(D\right)\right\rangle \circ\mathit{KCsA}\left(D\right)
\end{align*}
where $\mathit{Fse}\left(D,a,b\right)=\left(\mathit{SE}\left(D,a\right),\mathit{SE}\left(D,b\right)\right)$,
$\mathit{F!!}\left(D,a,b\right)=\left(!!\left(D,a\right),!!\left(D,b\right)\right)$.

Note that, the equality between the third equation and the last equation
is a assertion of fact, rather than a results can be proved (verified).
This equivalence comes from the fact that size $K$-combinations can
be constructed by joining all possible combinations of size $i$ and
size $K-i$ combinations, where $0\leq i\leq K$.

Next, we prove the two paths between $\left(\left(\mathit{Css},\mathit{Cs}\right),\left(\mathit{Css},\mathit{Cs}\right)\right)$
and $\left(\mathit{Css},\mathit{NCss}\right)$ are equivalent.

\begin{align*}
	& \texttt{\ensuremath{\left\langle \mathit{SE}\left(D\right)\circ\mathit{KCsA}\left(D\right)\circ\mathit{Ffst},\mathit{KcsA}\left(K\right)\circ\cup\circ\left\langle \mathit{KCs}\left(K\right)\circ!!\left(D\right)\circ\mathit{KCsA}\left(D\right)\circ\mathit{Ffst},\mathit{KCsA}\left(K\right)\circ\mathit{Fsnd}\right\rangle \right\rangle }}\circ\\
	& \qquad\qquad\left(\mathit{id}\times\mathit{KCs}\left(K\right)\right)\times\left(\mathit{id}\times\mathit{KCs}\left(K\right)\right)\\
	\equiv & \text{ Product fusion,\ensuremath{f\times g=\left\langle f\circ\mathit{Ffst},f\circ\mathit{Fsnd}\right\rangle }},\mathit{FKCssnd}\left(D,\left(a,b\right),\left(c,d\right)\right)=\left(\mathit{KCs}\left(D,b\right),\mathit{KCs}\left(D,d\right)\right)\\
	& \texttt{\ensuremath{\left\langle \mathit{SE}\left(D\right)\circ\mathit{KCsA}\left(D\right)\circ\mathit{Ffst},\mathit{KcsA}\left(K\right)\circ\cup\circ\left\langle \mathit{KCs}\left(K\right)\circ!!\left(D\right)\circ\mathit{KCsA}\left(D\right)\circ\mathit{Ffst},\mathit{KCsA}\left(K\right)\circ\mathit{FKCs}\left(D\right)\right\rangle \right\rangle }}\\
	\equiv & \text{ Definition of \ensuremath{\mathit{Kcombs}}}\\
	& \texttt{\ensuremath{\left\langle \mathit{SE}\left(D\right)\circ\mathit{KCsA}\left(D\right)\circ\mathit{Ffst},\mathit{KcsA}\left(K\right)\circ\cup\circ\left\langle \mathit{KCs}\left(K\right)\circ!!\left(D\right)\circ\mathit{KCsA}\left(D\right)\circ\mathit{Ffst},\mathit{KCs}\left(K\right)\circ\cup\circ\mathit{Fsnd}\right\rangle \right\rangle }}\\
	\equiv & \text{ Definition of product}\\
	& \texttt{\ensuremath{\left\langle \mathit{SE}\left(D\right)\circ\mathit{KCsA}\left(D\right)\circ\mathit{Ffst},\mathit{KcsA}\left(K\right)\circ\cup\circ\mathit{KCs}\left(K\right)\circ\left\langle !!\left(D\right)\circ\mathit{KCsA}\left(D\right)\circ\mathit{Ffst},\cup\circ\mathit{Fsnd}\right\rangle \right\rangle }}\\
	\equiv & \text{ Definition of \ensuremath{KCs}}\\
	& \texttt{\ensuremath{\left\langle \mathit{SE}\left(D\right)\circ\mathit{KCsA}\left(D\right)\circ\mathit{Ffst},\mathit{KCs}\left(K\right)\circ\cup\circ\left\langle !!\left(D\right)\circ\mathit{KCsA}\left(D\right)\circ\mathit{Ffst},\cup\circ\mathit{Fsnd}\right\rangle \right\rangle }}
\end{align*}

\subsection{Proof of fusion condition\label{subsec:Proof-of-fusion}}
\begin{lem}
	\emph{$\mathit{DeepICEAlg}$ satisfies the following fusion condition}
	
	\emph{
		\begin{equation}
			\mathit{DeepICE}\left(D,K\right)=f\circ\mathit{nestedCombsAlg}\left(D,K\right)=\mathit{DeepICEAlg}\left(D,K\right)\circ f\times f
		\end{equation}
		where $f=\mathit{min}_{\text{0-1}}\left(D\right)\circ\mathit{eval}^{\prime}\left(K-1\right)$,
		which defines the Deep ICE algorithm \ref{eq: deepice-definition}.}
\end{lem}
\begin{proof}
	For optimization problem, proving equality is often too strict that
	it rarely holds in practice. Instead, whenever a ``selector'' is used,
	we can relax the fusion condition by replacing the eqaulity as a set
	memership relation \citep{bird2020algorithm}.
	\begin{equation}
		f\circ\mathit{nestedCombsAlg}\left(D,K\right)\subseteq\mathit{DeepICEAlg}\left(D,K\right)\circ f\times f
	\end{equation}
	
	In point-wise style, this is equivalent to 
	\begin{equation}
		f\circ\mathit{nestedCombsAlg}\left(D,K,h\left(xs\right),h\left(ys\right)\right)\subseteq\mathit{DeepICEAlg}\left(D,K,f\left(h\left(xs\right)\right),f\left(h\left(ys\right)\right)\right)
	\end{equation}
	where $h\left(as\right)=\mathit{nestedCombs}\left(D,K,as\right)$.
	
	On the left side of the set membership relation, we first update the
	nested combinations by merging $\mathit{nestedCombs}\left(D,K,ys\right)$
	and $\mathit{nestedCombs}\left(D,K,ys\right)$ using $\mathit{\mathit{nestedCombsAlg}}$
	and then select the optimal $nc$ with respect to $E_{\text{0-1}}$
	by using $\mathit{min}_{\text{0-1}}\left(D\right)\circ\mathit{eval}^{\prime}\left(K-1\right)$.
	
	On the right-hand side, recall that $\mathit{nestedCombs}\left(D,K,as\right):\left[\mathbb{R}^{D}\right]\to\left(\left[\left[C\right]\right],\left[\left[NC\right]\right]\right)$
	returns all possible nested combinations ($K$-combination of hyperplanes)
	$ncss$, all possible combination of data items $css$ ($D$th inner
	list is empty) and $ncss$, and $f\circ h=\mathit{DeepICE}\left(D,K\right)$
	is the specification of the Deep-ICE algorithm. Functions $f\left(h\left(xs\right)\right)$
	and $f\left(h\left(ys\right)\right)$ select the optimal nested-combination
	with respect to $E_{\text{0-1}}$ from all possible nested combinations
	with respect to $xs$ and $ys$, call them $\mathit{optcnfg}_{1}$,
	and $\mathit{optcnfg}_{2}$ with respectively. Then the nested combinations
	are merged together and selected the new optimal configuration $\mathit{optcnfg}^{\prime}$
	by using $\mathit{DeepICEAlg}$. By definition, $\mathit{optcnfg}^{\prime}$
	is obtained by selection the optimal configurations from the newly
	generated combinations and compared with $\mathit{optcnfg}_{1}$,
	and $\mathit{optcnfg}_{2}$ , thus the solutions on the left side
	of the set membership relation must include in the right-hand side
	of the nested combination.
\end{proof}

\subsection{Algorithms}\label{subsec: algorithms}

Algorithm \ref{coreset-selection-method} present the recursive definition
of the Deep-ICE algorithm.

\begin{algorithm}[H]
	\textbf{Input}: $\mathit{ds}$: input data list; $D$: number of features;
	$K$: number of hyperplanes;\\
	
	\textbf{Output}: $\mathit{cnfg}:\left(\mathit{NC},\left\{ 1,-1\right\} ^{K}\right)$—Optimal
	nested combination with respect to $\mathit{ds}$; $\mathit{ncss}:\mathit{NCss}$—All
	possible nested combinations of size less than $K$; \emph{$\mathit{css}:\mathit{Css}$}—All
	possible combinations of size less than $D$.\emph{}\\
	
	\begin{raggedright}
		$\begin{aligned}\mathit{DeepICE} & \left(D,K,\left[\;\right]\right)=\mathit{nestedCombsAlg}_{1}\left(\text{\ensuremath{\left[\right]}}\right)\\
			\mathit{DeepICE} & \left(D,K,\left[a\right]\right)=\mathit{nestedCombsAlg}_{2}\left(\text{\ensuremath{\left[a\right]}}\right)\\
			\mathit{DeepICE} & \left(D,K,xs\cup ys\right)=min_{\text{0-1}}\left(K\right)\circ\mathit{eval}^{\prime}\left(K-1\right)\circ\\
			 &\mathit{nestedCombsAlg}_{3}\left(\mathit{DeepICE}\left(D,K,xs\right),\mathit{DeepICE}\left(D,K,ys\right)\right),
		\end{aligned}
		$\\
		\par\end{raggedright}
	where $\mathit{nestedCombsAlg}$ is defined as\\
	\\
	
	\begin{raggedright}
		$\begin{aligned}\mathit{nestedCombsAlg}_{1} & \left(d,k,\left[\;\right]\right)=\left(\left[\left[\left[\;\right]\right]\right],\left[\left[\left[\;\right]\right]\right]\right)\\
			\mathit{nestedCombsAlg}_{2} & \left(d,k,\left[x_{n}\right]\right)=\left(\left[\left[\left[\;\right]\right],\left[\left[x_{n}\right]\right]\right],\left[\left[\left[\;\right]\right]\right]\right)\\
			\mathit{nestedCombsAlg}_{3} & \left(d,k,\left(css_{1},ncss_{1}\right),\left(css_{1},ncss_{1}\right)\right)=\left(\mathit{setEmpty}\left(D,css\right),ncss\right).
		\end{aligned}
		$\\
		\par\end{raggedright}
	\begin{raggedright}
		where $css=\mathit{kcombsAlg}\left(D,css_{1},css_{2}\right)$, and
		$ncss$ is defined as\\
		\par\end{raggedright}
	$ncss=\begin{cases}
		\left[\left[\left[\;\right]\right]\right] & css!!\left(D\right)=\left[\;\right]\\
		\mathit{kcombsAlg}\left(K,\mathit{kcombsAlg}\left(K,ncss_{1},ncss_{2}\right),\mathit{kcombs}\left(K,css!!D\right)\right) & \text{otherwise}.
	\end{cases}$
	
	\caption{$\mathit{DeepICE}_{\text{rec}}$: DeepICE recursive definition \label{alg:Deep-ICE-algorithm}}
\end{algorithm}

We also provide both the pesudocode for the sequential version \ref{alg:Deep-ICE-sequential}
and D\&C versions \ref{alg:Deep-ICE-D=000026C} of the Deep-ICE algorithms.
\begin{algorithm}[H]
	\textbf{Input}: $\mathit{ds}$: input data list; $D$: number of features;
	$K$: number of hyperplanes;\\
	
	\textbf{Output}: $\mathit{cnfg}_{\text{opt}}:\left(\mathit{NC},\left\{ 1,-1\right\} ^{K}\right)$—Optimal
	nested combination with respect to $\mathit{ds}$; $l_{\text{opt}}$:
	optimal 0-1 loss, $\mathit{hyperAsgn}$: All possible predictions
	of hyperplanes with respect to input list; $\mathit{css}$: all possible
	nested combinations of size smaller than $D$ $\mathit{ncss}$: all
	possible nested combinations of size smaller than $K$;
	\begin{enumerate}
		\item $\mathit{css}=\left[\left[\left[\:\right]\right],\left[\right]^{k}\right]$
		// initialize combinations
		\item $\mathit{ncss}=\left[\left[\left[\:\right]\right],\left[\right]^{k}\right]$
		// initialize nested-combinations
		\item $\mathit{hyperAsgn}=\mathit{empty}\left(\left(\begin{array}{c}
			N\\
			D
		\end{array}\right),N\right)$ / initialize prediction of hyperplanes as a empty $\left(\begin{array}{c}
			N\\
			D
		\end{array}\right)\times N$ matrix
		\item $l_{\text{opt}}=N$ //initialize optimal 0-1 loss
		\item \textbf{for} $n\leftarrow\mathit{range}\left(0,N\right)$ \textbf{do}:
		//\textbf{$\mathit{range}\left(0,N\right)=\left[0,1,\ldots,N-1\right]$}
		\item $\quad$\textbf{for} $j\leftarrow\mathit{reverse}\left(\mathit{range}\left(D,n+1\right)\right)$
		\textbf{do}:
		\item $\quad$$\quad$$\mathit{updates}=\mathit{reverse}\left(\mathit{map}\left(\cup\mathit{ds}\left[n\right],\mathit{css}\left[j-1\right]\right)\right)$
		// the $\mathit{reverse}$ function is used to organize configurations
		in revolving door ordering
		\item $\quad$$\quad$$\mathit{css}\left[j\right]=\mathit{css}\left[j\right]\cup\mathit{updaets}$
		// update $\mathit{css}$ to generate combinations in revolving door
		ordering,
		\item $\quad$$\mathit{hyperAsgn}=\mathit{genModels}\left(\mathit{css}\left[D\right],\mathit{hyperAsgn}\right)$
		// generate positive/negative predictions for each hyperplane in $\mathit{css}\left[D\right]$
		\item $\quad$$\mathit{css}\left[D\right]=\left[\:\right]$ // empty $D$-combinations
		after generation
		\item $\quad$$C_{1}=\left(\begin{array}{c}
			n\\
			D-1
		\end{array}\right)$, $C_{2}=\left(\begin{array}{c}
			n\\
			D
		\end{array}\right)$
		\item $\quad$$\mathit{ncss}^{\prime}=\mathit{kcombs}\left(k,C_{2}-C_{1}\right)$
		\item $\quad$$\mathit{ncss}=\mathit{kcombsAlg}\left(K,\mathit{ncss},\mathit{ncss}^{\prime}\right)$
		\item $\quad$$\mathit{cnfg}^{\prime},l^{\prime}=\mathit{eval}\left(\mathit{ncss}\left[K\right],\mathit{hyperAsgn}\right)$
		// evaluate to the number of misclassification for each size $K$
		nested combination in $\mathit{ncss}\left[K\right]$
		\item $\quad$$\mathit{ncss}\left[K\right]=\left[\:\right]$ // empty size
		$K$ nested-combinations after evaluation
		\item $\quad$\textbf{if} $l^{\prime}\leq l_{\text{opt}}$:
		\item $\quad$$\quad$$l_{\text{opt}}=l^{\prime}$
		\item $\quad$$\quad$$\mathit{cnfg}_{\text{opt}}=\mathit{cnfg}^{\prime}$
		\item \textbf{return} $\mathit{cnfg}_{\text{opt}}$, $l_{\text{opt}}$,
		$\mathit{hyperAsgn}$, $\mathit{ncss}$, $\mathit{css}$
	\end{enumerate}
	\caption{$\mathit{DeepICE}_{\text{seq}}$: Deep-ICE sequential definition \label{alg:Deep-ICE-sequential}}
\end{algorithm}

\begin{algorithm}[H]
	\textbf{Input}: $\mathit{ds}$: input data list; $D$: number of features;
	$K$: number of hyperplanes;\\
	
	\textbf{Output}: $\mathit{cnfg}_{\text{opt}}:\left(\mathit{NC},\left\{ 1,-1\right\} ^{K}\right)$—Optimal
	nested combination with respect to $\mathit{ds}$; $l_{\text{opt}}$:
	optimal 0-1 loss
	\begin{enumerate}
		\item $\mathit{hyperAsgn}=\mathit{empty}\left(\left(\begin{array}{c}
			N\\
			D
		\end{array}\right),N\right)$ // initialize prediction of hyperplanes as a empty $\left(\begin{array}{c}
			N\\
			D
		\end{array}\right)\times N$ matrix
		\item $l_{\text{opt}}=N$ //initialize optimal 0-1 loss
		\item $\mathit{ds}_{i},\mathit{ds}_{j}=\mathit{splitToTwo\left(\mathit{ds}\right)}$//
		split the data set into two half
		\item parallel:
		\item $\quad$$\quad$ $\mathit{res}_{i}=\mathit{DeepICE}_{\text{seq}}\left(D,K,\mathit{ds}_{i}\right)$
		// Process first data list
		\item $\quad$$\quad$ $\mathit{res}_{j}=\mathit{DeepICE}_{\text{seq}}\left(D,K,\mathit{ds}_{j}\right)$
		// Process second data list
		\item sync // Wait for both tasks to complete
		\item // Retrieve results: configuration, loss, hyperplane assignments,
		combinations
		\item $\mathit{cnfg}_{i},l_{i},\mathit{\mathit{hyperAsgn}_{i}},\mathit{css}_{i},\mathit{ncss}_{i}=\mathit{res}_{i}$
		\item $\mathit{cnfg}_{j},l_{j},\mathit{\mathit{hyperAsgn}_{j}},\mathit{css}_{j},\mathit{ncss}_{j}=\mathit{res}_{j}$
		\item $\mathit{css},\mathit{ncss}=\mathit{nestedCombsAlg}_{3}\left(D,K,\left(\mathit{css}_{i},\mathit{ncss}_{i}\right),\left(\mathit{css}_{j},\mathit{ncss}_{j}\right)\right)$
		// Merge: Combine nested combinations from both subsets
		\item $\mathit{hyperAsgn}=\mathit{mergeAsgn}\left(\mathit{\mathit{hyperAsgn}_{i}},\mathit{\mathit{hyperAsgn}_{j}}\right)$
		// Merge hyperplane assignments
		\item $\mathit{cnfg}^{\prime},l^{\prime}=\mathit{eval}\left(\mathit{ncss}\left[K\right],\mathit{hyperAsgn}\right)$
		// Evaluate merged nested combinations for size K
		\item $\mathit{cnfgs}=\left[\left(\mathit{cnfg}_{i},l_{i}\right),\left(\mathit{cnfg}_{j},l_{j}\right),\left(\mathit{cnfg}^{\prime},l^{\prime}\right)\right]$
		// Collect all configurations and their losses
		\item $\left(\mathit{cnfg}_{\text{opt}},l_{\text{opt}}\right)=min_{\text{0-1}}\left(\left[\mathit{cnfgs}\right]\right)$
		// Select configuration with minimum 0-1 loss
		\item \textbf{return} $\mathit{cnfg}_{\text{opt}}$, $l_{\text{opt}}$
	\end{enumerate}
	\caption{$\mathit{DeepICE}_{\text{D\&C}}$: Deep-ICE divide-and-conquer definition
		\label{alg:Deep-ICE-D=000026C}}
\end{algorithm}

Algorithm \ref{coreset-selection-method} shows the structure of
the coreset selection method.

\begin{algorithm}[H]
	\begin{enumerate}
		\item \textbf{Input}: $\mathit{ds}$: input data list; $M$: Block size;
		$R$: number of shuffle time in each filtering process; $L$: Max-heap
		size; $B_{\max}$: Maximum input size for the Deep-ICE algorithm;
		$c\in\left(0,1\right]$: Shrinking factor for heap size
		\item \textbf{Output}: Max-heap containing top $L$ configurations and associated
		data blocks \emph{}\\
		\item Initialize coreset $\mathcal{C}\leftarrow\mathit{ds}$
		\item \textbf{while} $\mathcal{C}\leq B_{\max}$ \textbf{do}:
		\item $\quad$Reshuffle the data, divide $\mathcal{C}$ into $\left\lceil \frac{\left|\mathcal{C}\right|}{M}\right\rceil $
		blocks $\mathcal{C}_{B}=\left\{ C_{1},C_{2},\ldots,C_{\left\lceil \frac{\left|\mathcal{C}\right|}{M}\right\rceil }\right\} $
		\item $\quad$Initialize a size $L$ max-heap $\mathcal{H}_{L}$
		\item $\quad$\textbf{for} $r\leftarrow1$ \textbf{to $R$} \textbf{do}:
		\item $\quad$$\quad$$r=r+1$
		\item $\quad$$\quad$\textbf{for} $C\in\mathcal{C}_{B}$ \textbf{do}:
		\item $\quad\quad$$\quad$$\mathit{cnfg}\leftarrow\mathit{DeepICE}\left(D,K,C\right)$
		\item $\quad$$\quad$$\quad$$\mathcal{H}_{L}.\text{push}\left(\mathit{cnfg},C\right)$
		\item $\quad$ $\quad$$\mathcal{C}\leftarrow\mathit{unique}\left(\mathcal{H}_{L}\right)$
		// \emph{Merge blocks and remove duplicates}
		\item $\quad$$L\leftarrow L\times c$ // Shrink heap size:
		\item $\mathit{cnfg}\leftarrow\mathit{DeepICE}\left(D,K,\mathcal{C}\right)$
		//\emph{ Final refinement}
		\item $\mathcal{H}_{L}.\text{push}\left(\mathit{cnfg},\mathcal{C}\right)$
		\item \textbf{return} $\mathcal{H}_{L}$
	\end{enumerate}
	\caption{Deep-ICE with Coreset Filtering\label{coreset-selection-method}}
\end{algorithm}

\subsection{Complexity analysis \label{subsec:Complexity-analysis}}
\begin{thm}
	\emph{The DeepICE algorithm has a time complexity of $O\left(K\times N\times2^{K-1}\times\left(\begin{array}{c}
			\left(\begin{array}{c}
				N\\
				D
			\end{array}\right)\\
			K
		\end{array}\right)+N\times D^{3}\times\left(\begin{array}{c}
			N\\
			D
		\end{array}\right)\right)$ which is strictly smaller than $O\left(N^{DK+1}\right)$, and a space
		complexity of $O\left(\left(\begin{array}{c}
			\left(\begin{array}{c}
				N\\
				D
			\end{array}\right)\\
			K-1
		\end{array}\right)\times K+\left(\begin{array}{c}
			N\\
			D-1
		\end{array}\right)\times N\right)$, which is strictly smaller than $O\left(N^{D\left(K-1\right)}\right)$.}
\end{thm}
\begin{proof}
	We analyze the complexity using the sequential version of the DeepICE
	algorithm \ref{alg:Deep-ICE-sequential}. At stage $n$, the computation
	of lines 5–8 has complexity $O\left(n^{D-1}\right)$, since there
	are at most $\left(\begin{array}{c}
		n\\
		D-1
	\end{array}\right)$ new $D$-combinations in each recursive step. The computation at
	line 9 requires $O\left(n^{D-1}\times D^{3}\times N\right)$ time.
	Similarly, the new nested combinations at lines 12–14 has a size $O\left(\sum_{k=1}^{K}\left(\begin{array}{c}
		\left(\begin{array}{c}
			n\\
			D-1
		\end{array}\right)\\
		k
	\end{array}\right)\times\left(\begin{array}{c}
		\left(\begin{array}{c}
			n\\
			D
		\end{array}\right)\\
		k
	\end{array}\right)\right)$, which requires computations of a complexity $2^{K-1}\times N\times K$
	per nested combination, as each combination must evaluate $2^{K-1}$
	possible hyperplane orientations.
	
	By Vandermonde’s identity, we have 
	\[
	\sum_{k=1}^{K}\left(\begin{array}{c}
		\left(\begin{array}{c}
			n\\
			D-1
		\end{array}\right)\\
		k
	\end{array}\right)\times\left(\begin{array}{c}
		\left(\begin{array}{c}
			n\\
			D
		\end{array}\right)\\
		k
	\end{array}\right)=\left(\begin{array}{c}
		\left(\begin{array}{c}
			n\\
			D-1
		\end{array}\right)+\left(\begin{array}{c}
			n\\
			D
		\end{array}\right)\\
		k
	\end{array}\right)=\left(\begin{array}{c}
		\left(\begin{array}{c}
			n+1\\
			D-1
		\end{array}\right)\\
		k
	\end{array}\right)\leq\left(n+1\right)^{Dk}
	\]
	Summing over $n=0$ to $n=N-1$, the total time complexity becomes
	\begin{align*}
		& O\left(\sum_{n=0}^{N-1}\left(D^{3}\times N\times\left(\begin{array}{c}
			n\\
			D-1
		\end{array}\right)+K\times N\times2^{K-1}\times\left(\begin{array}{c}
			\left(\begin{array}{c}
				n+1\\
				D-1
			\end{array}\right)\\
			k
		\end{array}\right)\right)\right).\\
		= & O\left(N\times D^{3}\times\left(\begin{array}{c}
			N\\
			D
		\end{array}\right)+K\times N\times2^{K-1}\times\left(\begin{array}{c}
			\left(\begin{array}{c}
				N\\
				D
			\end{array}\right)\\
			K
		\end{array}\right)\right)\\
		\leq & O\left(N\times D^{3}\times N^{D}+K\times N\times2^{K-1}\times N^{DK}\right)\\
		= & O\left(N^{DK+1}\right)
	\end{align*}
	
	For memory, at lines 10 and 15, we clear the size $D$ combinations
	and size $K$ nested combinations, so we only need to store smaller
	configurations in memory. The resulting space complexity is
	\begin{equation}
		O\left(\left(\begin{array}{c}
			\left(\begin{array}{c}
				N\\
				D
			\end{array}\right)\\
			K-1
		\end{array}\right)\times K+\left(\begin{array}{c}
			N\\
			D
		\end{array}\right)\times N\right)=O\left(N^{D\left(K-1\right)}\right).
	\end{equation}
\end{proof}

\subsubsection{Ordered generation of combinations\label{subsec:Generating-combination-in order}}

To generate $D$-combinations of data points efficiently, we employ
a technique that organizes combinations in a specific order, assigning
each a unique ``rank.'' To achieve this, a critical but small function
$\mathit{reverse}$ used at line 6 of the DeepICE algorithm \ref{alg:Deep-ICE-sequential}
makes it possible. This allows $D$-combinations to be organized in
``revolving door ordering'' and thus combinations are represented
by their rank rather than storing the combinations explicitly. This
approach offers two key benefits: First, storing ranks significantly
reduces memory usage, from $M\times D\times64$ bits to $M\times\log\left(M\right)$bits
($\log\left(M\right)$ is often representable using 32 bits in coreset
selection method), where $M=\sum_{k=0}^{K-1}\left(\begin{array}{c}
	\begin{array}{c}
		N\\
		D
	\end{array}\\
	k
\end{array}\right)$. A workspace in memory is preallocated before training to store predictions
associated with these hyperplanes, thereby avoiding memory allocation
overhead during runtime. Second, it enables the organization of hyperplane
predictions into a $\left(\begin{array}{c}
	N\\
	D
\end{array}\right)\times N$ matrix, where each row corresponds to a unique rank. As a result,
the algorithm requires only $O\left(N\times D^{3}\times\left(\begin{array}{c}
	N\\
	D
\end{array}\right)\right)$ time. Moreover, storing hyperplanes in a single large matrix allows
exploitation of high-throughput hardware such as Nvidia GPU Tensor
Cores. Without this method, predictions would need to be recomputed
for each hyperplane, requiring at least $O\left(N\times D^{3}\times\left(\begin{array}{c}
	\left(\begin{array}{c}
		N\\
		D
	\end{array}\right)\\
	K
\end{array}\right)\right)$ time. This strategy reduces memory usage and accelerates execution
without drawbacks, and it can be extended to other problems involving
nested combinatorial structures.

\subsubsection{Memory-Free method by using unranking function}

Building on the first technique, the second method leverages the ordered
structure of $D$-combinations to eliminate the need to store $K$-combinations.
An unranking function takes the rank of a combination as input and
reconstructs the corresponding $K$-combination on demand. This supports
the dynamic generation of combinations for a given range of rank values,
thereby circumventing memory constraints that would otherwise limit
the algorithm due to insufficient storage. However, it incurs an additional
computational cost of $\Theta\left(K\right)$ arithmetic operations
per combination due to the unranking function. Despite this, the method
often improves overall efficiency by simplifying memory management,
leading to more effective implementations in practice.

However, this method has a limitation: it precludes the use of bounding
techniques because $K$-combinations combinations are reconstructed
on demand via unranking functions rather than stored in memory. If
future research requires such techniques, this approach is unsuitable,
as it is challenging to identify which configurations (represented
by ranks) are eliminated during algorithm execution.

\subsubsection{Empirical analysis\label{subsec:Empirical-analysis of run time}}
Figure \ref{fig:run-time-polynomial} shows that the empirical running time of the DeepICE algorithm aligns with the expected worst-case complexity.

\begin{figure}
	\begin{centering}
		\includegraphics[scale=0.25]{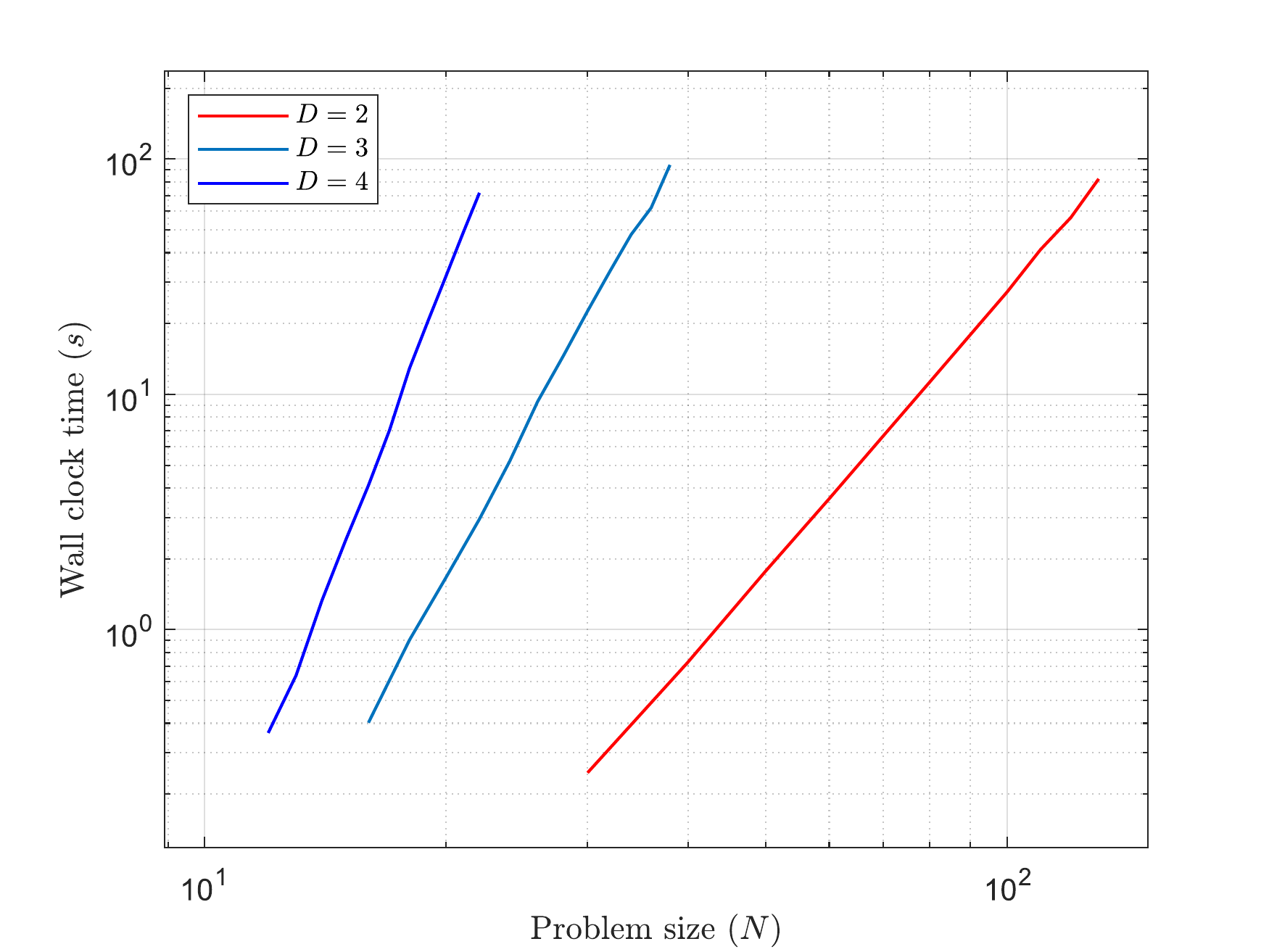}\includegraphics[scale=0.25]{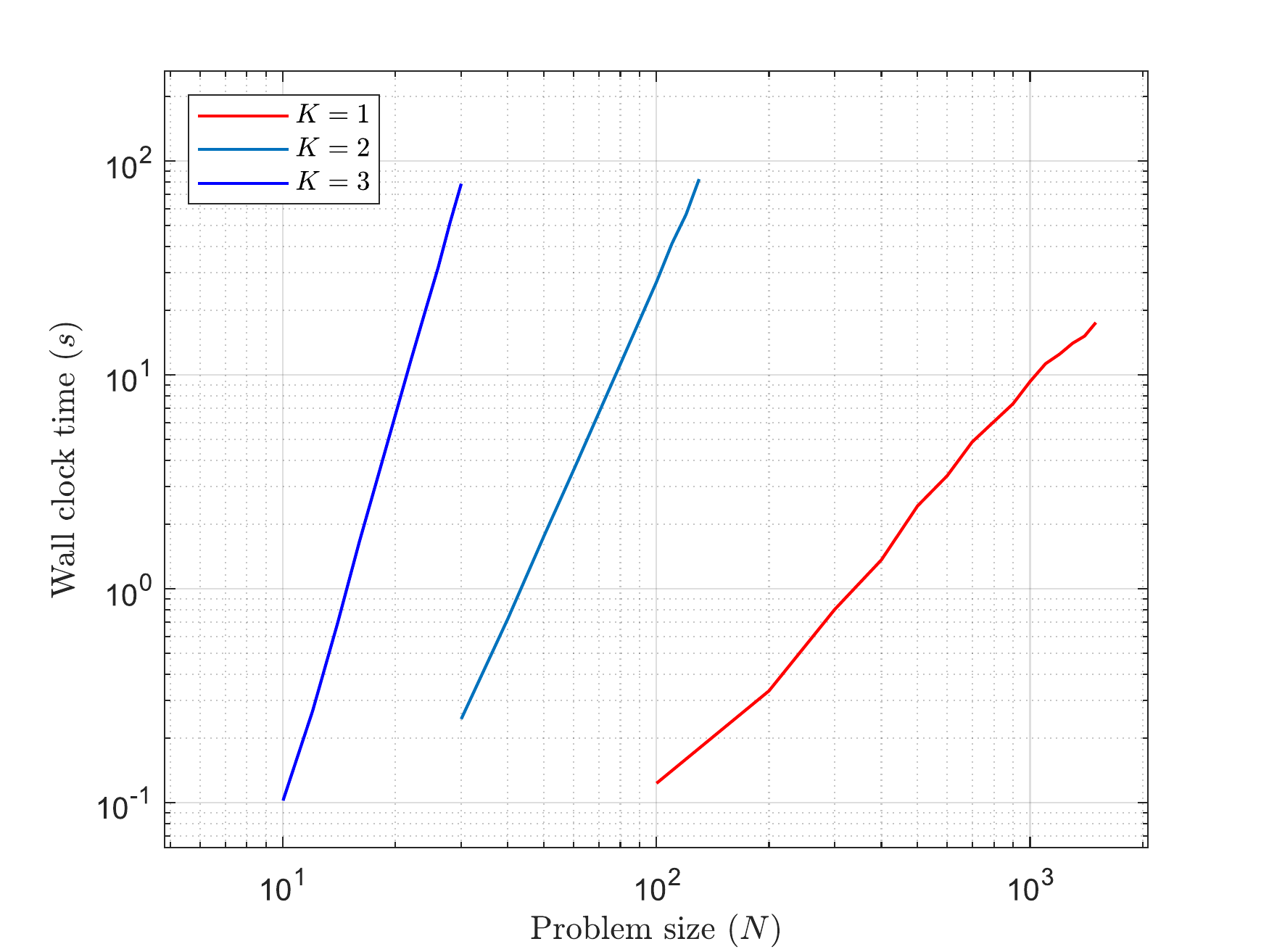}
		\par\end{centering}
	\caption{Empirical analysis shows that the wall-clock runtime of the DeepICE
		algorithm is strictly smaller than the predicted worst-case complexity
		$O\left(N^{DK+1}\right)$. The log-log wall-clock runtime (seconds)
		of DeepICE on synthetic datasets is plotted against dataset size $N$.
		On this log-log scale polynomial run time appears as a linear function
		of problem size $N$, and the slope of the line corresponds to the
		polynomial degree. In the left panel, the runtime curves (from left
		to right) correspond to $K=2$ with $D=2,3,4$, and have slopes 3.96,
		6.28, and 8.88—smaller than the predicted worst-case exponents $O\left(N^{4}\right)$,
		$O\left(N^{7}\right)$, $O\left(N^{9}\right)$. In the right panel,
		the curves (from left to right) correspond to $D=2$ with $K=1,2,3$
		respectively), and have slopes 1.91, 3.95, and 6.11—smaller than the
		predicted worst-case exponents $O\left(N^{3}\right)$, $O\left(N^{5}\right)$,
		$O\left(N^{7}\right)$, respectively,. \label{fig:run-time-polynomial}}
	
\end{figure}

\subsection{Additional experiments \label{subsec:Additional-experiments}}

\subsubsection{Comparison with expected accuracy optimization (EXACT) framework}

\begin{table}
	\scriptsize
	\caption{Five-fold cross-validation results on the UCI dataset. We compare
		the performance of our Deep-ICE algorithm, with $K$ (number of hyperplanes)
		ranging from 1 to 3, trained either with the coreset selection method or directly (marked by {*})—against
		\citet{karpukhin2024exact}'s expected accuracy optimization (EXACT)
		framework. Results are reported as mean accuracy loss over training
		and test sets in the format: Training Error / Test Error (Standard
		Deviation: Train / Test). The best-performing algorithm in each row
		is highlighted in bold.\label{tab: compare with EXACT}}
	
	\begin{center} 
		\begin{tabular}{@{} 
				>{\raggedright}p{0.03\textwidth}  
				>{\raggedleft}p{0.03\textwidth} 
				>{\raggedleft}p{0.02\textwidth} 
				>{\raggedleft}p{0.1\textwidth} 
				>{\raggedleft}p{0.1\textwidth} 
				>{\raggedleft}p{0.1\textwidth} 
				>{\raggedleft}p{0.09\textwidth} 
				>{\raggedleft}p{0.09\textwidth} 
				>{\raggedleft}p{0.1\textwidth} 
				@{}}
			
			Dataset & $N$ & $D$ & Deep-ICE (\%) ($K=1$) &  Deep-ICE (\%) ($K=2$) &  Deep-ICE (\%) ($K=3$) & EXACT (\%) ($K=1$) &  EXACT (\%) ($K=2$)  &  EXACT (\%) ($K=3$)  			\tabularnewline
			\midrule  
			\addlinespace[0.5cm]
			Ai4i & 10000 & 6 &  
				\centering
				97.45/97.40
				
				(0.10/0.36)
			 
				&
				\textbf{97.90}/\textbf{97.82}
				
				(0.01/0.35)
			 
				&
				97.71/97.71
				
				(0.10/0.25)
			 
				&
				96.61/96.61
				
				(0.01/0.02)
			 
				&
				96.63/96.60
				
				(0.04/0.03)
			 
				&
				96.69/96.62
				
				(0.10/0.05)
		\tabularnewline
 
			Caesr & 72 & 5 &  
				\centering
				{*}74.55/82.67
				
				(7.18/16.11)
			 
				&
				\textbf{89.45}/\textbf{88.00}
				
				(4.21/9.80)
			 
				&
				84.36/86.67
				
				(7.51/5.96)
			 
				&
				79.50/69.24
				
				(2.44/12.59)
			 
				&
				81.94/62.38
				
				(2.65/10.03)
			 
				&
				87.83/64.00
				
				(2.74/8.72)
			 \tabularnewline
			VP & 704 & 2 &  
				\centering
				{*}96.94/\textbf{97.59}
				
				(0.44/1.46)
			 
				&
				97.76/\textbf{97.59}
				
				(0.41/1.65)
			 
				&
				\textbf{97.80}/97.45
				
				(0.43/1.71)
			 
				&
				92.47/92.47
				
				(0.08/0.34)
			 
				&
				92.47/92.47
				
				(0.08/0.34)
			 
				&
				92.47/92.47
				
				(0.08/0.34)
			 \tabularnewline
 
			Spesis & 975 & 3 &  
				\centering
				{*}94.47/92.88
				
				(0.10/0.61)
			 
				&
				\textbf{96.43}/95.26
				
				(0.49/1.82)
			 
				&
				96.24/\textbf{95.36}
				
				(0.22/1.62)
			 
				&
				94.05/94.05
				
				(0.06/0.25)
			 
				&
				94.05/94.05
				
				(0.06/0.25)
			 
				&
				94.05/94.05
				
				(0.06/0.25)
			 \tabularnewline
 
			HB & 283 & 3 &  
				\centering
				{*}77.18/75.44
				
				(0.45/2.48)
			 
				&
				80.11/77.19
				
				(0.74/2.48)
			 
				&
				\textbf{80.85}/\textbf{78.53}
				
				(1.02/3.57)
			 
				&
				75.70/71.39
				
				(1.94/3.86)
			 
				&
				77.74/73.45
				
				(0.55/5.89)
			 
				&
				78.71/71.36
				
				(1.66/3.56)
			 \tabularnewline
 
			BT & 502 & 4 &  
				\centering
				{*}77.13/76.36
				
				(1.46/2.71)
			 
				&
				\textbf{79.59}/\textbf{77.98}
				
				(0.62/3.38)
			 
				&
				79.36/77.98
				
				(0.59/2.88)
			 
				&
				77.84/73.51
				
				(1.31/2.80)
			 
				&
				78.09/73.51
				
				(1.54/2.36)
			 
				&
				78.14/73.11
				
				(1.56/3.09)
			 \tabularnewline
 
			AV & 2342 & 7 &  
				\centering
				89.89/88.52
				
				(0.33/1.56)
			 
				&
				\textbf{90.34}/89.04
				
				(0.15/1.39)
			 
				&
				89.77/88.76
				
				(0.33/1.75)
			 
				&
				87.18/87.18
				
				(0.03/0.11)
			 
				&
				87.26/87.03
				
				(0.19/0.37)
			 
				&
				87.70/87.13
				
				(0.34/0.40)
			 \tabularnewline
 
			SO & 1941 & 27 &  
				\centering
				\textbf{77.77}/\textbf{76.03}
				
				(0.43/0.83)
			 
				&
				77.13/75.33
				
				(0.81/1.32)
			 
				&
				76.66/74.95
				
				(0.74/1.38)
			 
				&
				76.33/73.11
				
				(0.26/1.82)
			 
				&
				78.81/75.22
				
				(1.68/1.77)
			 
				&
				79.95/75.32
				
				(2.16/2.31)
			 \tabularnewline
 
			DB & 1146 & 9 &  
				\centering
				78.78/79.69
				
				(0.41/0.69)
			 
				&
				83.60/\textbf{81.37}
				
				(0.43/2.52)
			 
				&
				\textbf{83.88}/81.32
				
				(0.98/2.23)
			 
				&
				73.91/70.59
				
				(1.15/3.52)
			 
				&
				73.91/68.93
				
				(3.83/4.31)
			 
				&
				75.94/72.42
				
				(1.73/2.30)
			 \tabularnewline
 
			RC & 3810 & 7 &  
				\centering
				93.88/92.45
				
				(0.28/1.02)
			 
				&
				\textbf{93.91}/\textbf{93.10}
				
				(0.24/1.02)
			 
				&
				93.94/92.98
				
				(0.21/0.98)
			 
				&
				93.03/92.62
				
				(0.32/1.08)
			 
				&
				93.08/92.60
				
				(0.29/1.16)
			 
				&
				93.08/92.76
				
				(0.35/1.24)
			 \tabularnewline
 
			SS & 51433 & 3 & 
				\centering
				86.57/86.72
				
				(0.03/0.15)
			 
				&
				\textbf{86.60}/\textbf{86.72}
				
				(0.04/0.16)
			 
				&
				86.59/86.70
				
				(0.03/0.11)
			 
				&
				86.57/86.53
				
				(0.04/0.2)
			 
				&
				86.55/86.53
				
				(0.05/0.17)
			 
				&
				86.56/86.53
				
				(0.05/0.17)
\tabularnewline

		\end{tabular}
\par\end{center}
\end{table}

This Subsection we compared with the expected accuracy optimization
(EXACT) method proposed by \citet{karpukhin2024exact} the results
is shown in table \citet{karpukhin2024exact}.

\subsubsection{Wall-clock run time comparison}

\begin{table}
	\scriptsize
	\caption{Running time (seconds) of each algorithm, with ``0.01$<$'' denotes a
		time smaller than 0.01 seconds. In principle, allocating more computational
		resources to DeepICE yields better solutions. For comparison, we record
		the wall-clock time at which DeepICE first obtains a solution with
		lower 0–1 loss than the other methods. The reported times are the
		medians over three runs. \label{tab:Running-time-(seconds)}}
	
	\begin{center} 
		\begin{tabular}{@{} 
				>{\raggedright}p{0.03\textwidth}  
				>{\raggedleft}p{0.03\textwidth} 
				>{\raggedleft}p{0.02\textwidth} 
				>{\raggedleft}p{0.06\textwidth} 
				>{\raggedleft}p{0.06\textwidth} 
				>{\raggedleft}p{0.06\textwidth} 
				  >{\raggedleft}p{0.04\textwidth} 
				>{\raggedleft}p{0.04\textwidth} 
				>{\raggedleft}p{0.04\textwidth} 
				>{\raggedleft}p{0.04\textwidth}
				>{\raggedleft}p{0.04\textwidth} 
				>{\raggedleft}p{0.04\textwidth} 
				>{\raggedleft}p{0.04\textwidth} 
				@{}}
			
			Dataset & $N$ & $D$ & Deep-ICE (s) ($K=1$) &  Deep-ICE (s) ($K=2$) &  Deep-ICE (s) ($K=3$) &SVM (s) & MLP (s) ($K=1$) &  MLP (s) ($K=2$)  &  MLP (\%) ($K=3$) & EXACT (s) ($K=1$) &  EXACT (s) ($K=2$)  &  EXACT (s) ($K=3$)  		\tabularnewline
			\midrule  
Ai4i & 10000 & 6 & 450.5 & 622.42 & 505.42 & 0.05 & 18.43 & 18.74 & 16.52 & 22.56 & 21.91 & 26.47\tabularnewline
Caesr & 72 & 5 & 0.26 & 0.26 & 7.10 & 0.01$<$ & 12.27 & 12.96 & 10.11 & 23.78 & 22.92 & 29.80\tabularnewline
VP & 704 & 2 & 0.83 & 0.45 & 0.85 & 0.01$<$ & 11.19 & 11.53 & 10.76 & 24.36 & 23.49 & 25.22\tabularnewline
Spesis & 975 & 3 & 8.00 & 0.21 & 0.41 & 0.01$<$ & 9.56 & 10.60 & 11.68 & 25.24 & 23.01 & 29.93\tabularnewline 
HB & 283 & 3 & 0.20 & 0.21 & 0.38 & 0.01$<$ & 12.54 & 14.43 & 17.68 & 22.65 & 23.01 & 25.30\tabularnewline
BT & 502 & 4 & 0.26 & 0.36 & 0.43 & 0.01$<$ & 13.45 & 12.34 & 15.24 & 25.46 & 24.36 & 27.99\tabularnewline 
AV & 2342 & 7 & 132.51 & 294.81 & 356.51 & 0.01$<$ & 14.53 & 14.12 & 13.21 & 22.86 & 23.02 & 28.19\tabularnewline
SO & 1941 & 27 & 762.50 & 850.42 & 543.54 & 0.04 & 12.43 & 13.23 & 14.53 & 24.12 & 25.61 & 29.36\tabularnewline
DB & 1146 & 9 & 50.43 & 20.39 & 16.77 & 0.01 & 14.78 & 16.20 & 17.43 & 26.45 & 22.57 & 27.35\tabularnewline
RC & 3810 & 7 & 423.5 & 217.26 & 611.37 & 0.02 & 15.02 & 17.02 & 14.53 & 24.68 & 22.77 & 27.21\tabularnewline
SS & 51433 & 3 & 1.13 & 3.26 & 4.21 & 9.63 & 43.19 & 73.19 & 77.43 & 25.94 & 21.86 & 26.00\tabularnewline

		\end{tabular}
		\par\end{center}
\end{table}

Table \ref{tab:Running-time-(seconds)} report the run-time comparison
of between DeepICE, SVM, MLP and EXACT.

\subsection{Experiments of exhaustively exploring all solutions\label{subsec:Experiments-of-exhuastively}}

For $K=1$ case, i.e., linear case, the Deep-ICE algorithm fully explores
the solution space for datasets such as Voicepath, Caesarian, Sepsis,
HB, and BT. We output all solutions whose training accuracy is lower
than that of the SVM. The regularization parameter for the SVM is
fixed at 1 across all datasets. We deliberately avoid tuning this
parameter to achieve the lowest test error, as a solution with lower
test accuracy may increase training error, thereby generating more
candidate solutions due to the higher training error. Adjusting the
regularization parameter introduces a trade-off between training and
test errors, complicating the analysis. To keep our discussion focused
and consistent, we fix the regularization parameter. We summarize
the empirical results in Table \ref{subsec:Experiments-of-exhuastively}.
Using the generated solutions, we construct hyperplanes and select
two representative types from each equivalence class: (1) hyperplanes
passing through exactly $D$ points (direct hyperplanes), and (2)
arbitrary hyperplanes computed via linear programming (LP hyperplanes).
In the table, we compare the out-of-sample performance of these solutions
against that of the SVM. We found no strong evidence that the maximal-margin
hyperplane (SVM) consistently outperforms other hyperplanes with lower
training errors. For example, in the HB dataset, an average of 8,922.2
solutions outperform the SVM in training dataset. Of these, direct
hyperplanes have an average of 5,448.2 solutions, and LP hyperplanes
have an average of 6,165.6 solutions, outperforming the SVM in out-of-sample
test.

\begin{table}[h]
	\centering
	\caption{Comparing the average out-of-sample accuracy in a 5-fold cross-validation.
		All solutions with training accuracy lower than that of the SVM are
		generated, and their total number is reported (Total number of solutions).
		Two representative hyperplanes from the equivalence classes are included:
		the \emph{direct} \emph{hyperplane}, which passes through exactly
		$D$ points, and the \emph{LP} \emph{hyperplane}, computed via linear
		programming. For each type, the average number of hyperplanes with
		out-of-sample accuracy lower than that of the SVM is also reported.
		}
		\label{tab: exhaustive-generation}
	\begin{tabular}{c c c c }

		Datasets & Total number of solutions & Direct hyperplanes & LP hyperplanes
		\\ \hline \\ 

		Caesarian & 4430.2 & 2379.4 & 2913.8 \\
		Voicepath & 124.2 & 55.2 & 54.8  \\

		Spesis & 5.8 & 1 & 4.6  \\

		HB & 8922.2 & 5448.2 & 6165.6 \\

		BT & 5150.4 & 3189.8 & 3580 \\

	\end{tabular}
\end{table}

\end{document}